%% file: al_linsep.tex
\documentclass{article} % Anonymized submission
\usepackage[numbers]{natbib}

\usepackage[final]{nips_2017}
\usepackage[utf8]{inputenc} % allow utf-8 input
\usepackage[T1]{fontenc}    % use 8-bit T1 fonts
\usepackage{hyperref}       % hyperlinks
\usepackage{url}            % simple URL typesetting
\usepackage{booktabs}       % professional-quality tables
\usepackage{amsfonts}       % blackboard math symbols
\usepackage{nicefrac}       % compact symbols for 1/2, etc.
\usepackage{microtype}      % microtypography

\usepackage{amsmath}
\usepackage{amssymb}
\usepackage{amsthm}
\usepackage{algorithm}
\usepackage{algorithmic}
\usepackage[usenames,dvipsnames]{pstricks}
\usepackage{pst-grad} % For gradients
\usepackage{pst-plot} % For axes
\usepackage{dsfont}
\usepackage{commath}
\usepackage{xspace}
\usepackage{color}
\usepackage{times}
\usepackage{enumitem}

\newtheorem{theorem}{Theorem}
\newtheorem{lemma}{Lemma}
\newtheorem{claim}{Claim}
\newtheorem{corollary}{Corollary}
\newtheorem{definition}{Definition}

\renewcommand\cbr[1]{\left\{#1\right\}}
\def\S{\mathbb{S}}
\def\R{\mathbb{R}}
\def\E{\mathbb{E}}
\def\P{\mathbb{P}}
\def\Q{\mathbb{Q}}
\def\one{\mathds{1}}
\def\calX{\mathcal{X}}
\def\calY{\mathcal{Y}}
\def\calH{\mathcal{H}}
\def\calO{\mathcal{O}}
\def\calA{\mathcal{A}}
\def\calW{\mathcal{W}}
\def\calV{\mathcal{V}}

\def\poly{\text{poly}}
\def\superpoly{\text{superpoly}}
\def\d{\text{d}}
\DeclareMathOperator{\B}{B}

\DeclareMathOperator{\sign}{sign}
\DeclareMathOperator{\geom}{Geometric}
\DeclareMathOperator{\ber}{Bernoulli}
\DeclareMathOperator\err{err}

\newcommand\AP{\ensuremath{\operatorname{\textsc{Active-Perceptron}}}\xspace}
\newcommand\MP{\ensuremath{\operatorname{\textsc{Modified-Perceptron}}}\xspace}
\newcommand\PP{\ensuremath{\operatorname{\textsc{Passive-Perceptron}}}\xspace}
\newcommand\PMP{\ensuremath{\operatorname{\textsc{Passive-Modified-Perceptron}}}\xspace}

\newtheorem{innercustomthm}{Theorem}
\newenvironment{customthm}[1]
  {\renewcommand\theinnercustomthm{#1}\innercustomthm}
  {\endinnercustomthm}

\newcommand\hide[1]{}

 \author{
 Songbai Yan\\
 UC San Diego\\
 La Jolla, CA\\
 \texttt{yansongbai@ucsd.edu}
 \And
 Chicheng Zhang\thanks{Work done while at UC San Diego.}\\
 Microsoft Research\\
 New York, NY \\
\texttt{chicheng.zhang@microsoft.com}\\
 }

\title{Revisiting Perceptron: \\Efficient and Label-Optimal Learning of Halfspaces}

\begin{document}

\maketitle

\begin{abstract}
%\cz{There is no "active" mentioned - but we start with active in the introduction - perhaps we should fix it?}
It has been a long-standing problem to efficiently learn a halfspace using as few labels as possible in the presence of noise. In this work, we propose an efficient Perceptron-based algorithm for actively learning homogeneous halfspaces under the uniform distribution over the unit sphere. Under the bounded noise condition~\cite{MN06}, where each label is flipped with probability at most $\eta < \frac 1 2$, our algorithm achieves a near-optimal label complexity of $\tilde{O}\left(\frac{d}{(1-2\eta)^2}\ln\frac{1}{\epsilon}\right)$\footnote{We use $\tilde{O}(f(\cdot)) := O(f(\cdot)\ln f(\cdot))$, and $\tilde{\Omega}(f(\cdot)) := \Omega(f(\cdot) / \ln f(\cdot))$.
We say $f(\cdot) = \tilde{\Theta}(g(\cdot))$ if $f(\cdot) = \tilde{O}(g(\cdot))$ and $f(\cdot) = \tilde{\Omega}\left(g(\cdot)\right)$} in time $\tilde{O}\left(\frac{d^2}{\epsilon(1-2\eta)^3}\right)$.
Under the adversarial noise condition~\cite{ABL14, KLS09, KKMS08}, where at most a $\tilde \Omega(\epsilon)$ fraction of labels can be flipped, our algorithm achieves a near-optimal
label complexity of
$\tilde{O}\left(d\ln\frac{1}{\epsilon}\right)$ in time $\tilde{O}\left(\frac{d^2}{\epsilon}\right)$. Furthermore, we show
that our active learning algorithm can be converted to an efficient passive learning algorithm that has near-optimal sample complexities
with respect to $\epsilon$ and $d$.
\end{abstract}

%\begin{keywords}
%List of keywords
%\end{keywords}

\input{intro}
\input{relwork}
\input{setting}
\input{algorithm}
\input{theory}
\input{passive}
\paragraph{Acknowledgments.}
%The The authors are grateful to Kamalika Chaudhuri for help and support in this paper.
%Thanks are also due to the anonymous reviewers for their helpful feedback.
The authors thank Kamalika Chaudhuri for help and support, Hongyang Zhang for thought-provoking initial conversations, Jiapeng Zhang for helpful discussions, and the anonymous reviewers for their insightful feedback.
Much of this work is supported by NSF IIS-1167157 and 1162581.

\newpage
\bibliographystyle{plainnat}
\bibliography{alsearch}

\newpage
\input{appendix}

\end{document}

%% file: intro.tex
\section{Introduction}

We study the problem of designing efficient noise-tolerant algorithms for
actively learning homogeneous halfspaces in the streaming setting.
We are given access to a data distribution from which we can draw unlabeled examples, and a noisy labeling oracle
$\calO$ that we can query for labels. The goal is to find a computationally efficient algorithm to learn a halfspace that best classifies the data while making as few queries to the labeling oracle as possible.

%The labels returned by the oracle may be noisy.

Active learning arises naturally in many machine learning applications where unlabeled examples are abundant and cheap, but labeling requires human effort and is expensive. For those applications, one natural question is whether we can learn an accurate classifier using as few labels as possible. Active learning addresses this question by allowing the learning algorithm to sequentially select examples to query for labels, and avoid requesting labels which are less informative, or can be inferred from previously-observed examples.

There has been a large body of work on the theory of active learning,
showing sharp distribution-dependent label complexity bounds~\citep{CAL94, BBL09, H07, DHM07, H09, K10, ZC14, HAHLS15}.
However, most of these general active learning algorithms rely on solving empirical risk minimization problems, which are computationally hard in the presence of noise~\citep{ABSZ93}.

%A line of work considers efficient learning of linear separators with noise
On the other hand, existing computationally efficient algorithms for learning halfspaces~\citep{BFKV98, DV04, KKMS08, KLS09, ABL14, D15, ABHU15, ABHZ16} are not optimal in terms of label requirements.
These algorithms have different degrees of noise tolerance (e.g. adversarial noise~\cite{ABL14},
malicious noise~\cite{KL93}, random classification noise~\cite{AL88}, bounded noise~\cite{MN06}, etc), and run in time
polynomial
 in $\frac{1}{\epsilon}$ and $d$. Some of them naturally exploit the utility of active learning~\cite{ABL14, ABHU15, ABHZ16},
but they do not achieve the sharpest label complexity bounds in contrast to those computationally-inefficient active learning algorithms~\cite{BBZ07,BL13,ZC14}.

Therefore, a natural question is: is there any active learning halfspace algorithm
that is computationally efficient, and has a minimum
 label requirement? This has been posed as an open problem in~\cite{M06}.
 In the realizable setting, ~\cite{DKM05, BBZ07, BL13, TD17} give efficient algorithms that have optimal label complexity of $\tilde{O}(d\ln\frac1\epsilon)$ under some distributional assumptions. However, the challenge still remains open in the nonrealizable setting. It has been shown that learning halfspaces with agnostic noise even under Gaussian unlabeled distribution is hard \cite{KK14}. Nonetheless, we give an affirmative answer to this question under two moderate noise settings: bounded noise and adversarial noise.
 %Since it has been shown that learning halfspaces with agnostic noise even under Gaussian unlabeled distribution is hard \cite{KK14}, we restrict our attention to under which conditions learning can be done efficiently with optimal label requirements. \cz{Is this sentence correct?} Our paper gives an affirmative answer to this question under two moderate noise settings: bounded noise and adversarial noise.

%Note that~\cite{BBZ07, BL13} also has a similar algorithm, but with a

\subsection{Our Results}
We propose a Perceptron-based algorithm, \AP, for actively learning homogeneous halfspaces under the uniform distribution over the unit sphere. It works under two noise settings: bounded noise and adversarial noise. Our work answers an open question by~\cite{DKM05} on whether Perceptron-based active learning algorithms can be modified to tolerate label noise.

In the \emph{$\eta$-bounded noise setting} (also known as the Massart noise model~\citep{MN06}), the label of an example $x\in\R^d$ is generated by $\sign(u \cdot x)$ for some underlying halfspace $u$, and flipped with probability $\eta(x) \leq \eta < \frac 1 2$. Our algorithm runs in time $\tilde{O}\del{\frac{d^2}{(1-2\eta)^3 \epsilon} }$, and requires
$\tilde{O}\del{\frac{d}{(1-2\eta)^2} \cdot \ln\frac{1}{\epsilon}}$ labels. We show that this label complexity is \emph{nearly optimal} by providing an almost matching information-theoretic lower bound of $\Omega\left(\frac{d}{(1-2\eta)^{2}} \cdot \ln\frac1{\epsilon} \right)$. Our time and label complexities substantially improve over the state of the art result of~\cite{ABHZ16}, which runs in time
$\tilde O(d^{O(\frac 1 {(1-2\eta)^4})} \frac1\epsilon)$ and requires $\tilde O(d^{O(\frac 1 {(1-2\eta)^4})} \ln\frac1\epsilon)$ labels.

Our main theorem on learning under bounded noise is as follows:

\begin{customthm}{\ref{thm:ap-bn}}[Informal]
Suppose the labeling oracle $\calO$ satisfies the $\eta$-bounded noise condition with respect to $u$,
then for \AP, with probability at least $1-\delta$:
(1) The output halfspace $v$ is such that $\P[\sign(v \cdot X) \neq \sign(u \cdot X)] \leq \epsilon$;
(2) The number of label queries to oracle $\calO$ is at most $\tilde{O}\del{\frac{d}{(1-2\eta)^2} \cdot \ln\frac{1}{\epsilon} }$;
(3) The number of unlabeled examples drawn is at most $\tilde{O}\del{\frac{d}{(1-2\eta)^3\epsilon}}$;
(4) The algorithm runs in time $\tilde{O}\del{\frac{d^2}{(1-2\eta)^3\epsilon}}$.
\end{customthm}

In addition, we show that our algorithm also works in a more challenging setting,
the \emph{$\nu$-adversarial noise setting}~\cite{ABL14,KKMS08,KLS09}.\footnote{Note that the adversarial noise model is not the same as that in online learning~\cite{CBL06}, where each example can be chosen adversarially.} In this setting, the examples still come iid from a distribution, but the assumption on the labels is just that $\P[\sign(u \cdot X) \neq Y]\leq \nu$ for some halfspace $u$. Under this assumption, the Bayes classifier may not be a halfspace.
We show that our algorithm achieves an error of $\epsilon$ while tolerating a noise level of
$\nu = \Omega\left(\frac{\epsilon}{\ln \frac d \delta  + \ln\ln\frac1\epsilon}\right)$.
It runs in time $\tilde{O}\del{\frac {d^2} \epsilon}$,
and requires only $\tilde{O}\del{d \cdot \ln\frac{1}{\epsilon} }$ labels which is near-optimal.
\AP has a label complexity bound that matches the state of the art result of~\cite{HKY15}\footnote{The label complexity bound is implicit in~\cite{HKY15} by a refined analysis of the algorithm of~\cite{ABL14} (See their Lemma 8 for details).}, while having a lower running time.
%This matches the label complexity in the best known result , and runs faster than \cite{HKY15}.

Our main theorem on learning under adversarial noise is as follows:

\begin{customthm}{\ref{thm:ap-an}}[Informal]
Suppose the labeling oracle $\calO$ satisfies the $\nu$-adversarial noise condition with respect to $u$, where $\nu < \Theta(\frac{\epsilon}{\ln \frac d \delta  + \ln\ln\frac1\epsilon})$.
Then for \AP, with probability at least $1-\delta$:
(1) The output halfspace $v$ is such that $\P[\sign(v \cdot X) \neq \sign(u \cdot X)] \leq \epsilon$;
(2) The number of label queries to oracle $\calO$ is at most $\tilde{O}\del{d \cdot \ln\frac{1}{\epsilon} }$;
(3) The number of unlabeled examples drawn is at most $\tilde{O}\del{\frac d \epsilon}$;
(4) The algorithm runs in time $\tilde{O}\del{\frac {d^2} \epsilon}$.
\end{customthm}

%Our analysis is performed
Throughout the paper, \AP is shown to work if the unlabeled examples are drawn uniformly from the unit sphere. The algorithm and analysis can be easily generalized to any spherical symmetrical distributions, for example, isotropic Gaussian distributions. They can also be generalized to distributions whose densities with respect to uniform distribution are bounded away from 0. %\cz{Do we really want to say this in the final version? Sounds a bit silly.}

In addition, we show in Section~\ref{sec:passive} that \AP can be converted to a passive learning algorithm, \PP, that has near optimal sample complexities with respect to $\epsilon$ and $d$ under the two noise settings. We defer the discussion to the end of the paper.

\begin{table}
\centering
\caption{A comparison of algorithms for active learning of halfspaces under the uniform distribution, in the $\eta$-bounded noise model.}
\label{tab:comp-bn}
\begin{tabular}{lll}
\toprule
Algorithm & Label Complexity & Time Complexity \\
\midrule
\cite{BBZ07,BL13,ZC14} & $\tilde{O}(\frac{d}{(1-2\eta)^2}\ln\frac1\epsilon)$ & $\superpoly(d,\frac1\epsilon)$ \footnotemark\\
\cite{ABHZ16} & $\tilde{O}(d^{O(\frac 1 {(1-2\eta)^4})}\cdot\ln\frac1\epsilon)$ & $\tilde{O}(d^{O(\frac 1 {(1-2\eta)^4})}\cdot\frac1\epsilon)$ \\
Our Work & $\tilde{O}(\frac{d}{(1-2\eta)^2}\ln\frac1\epsilon)$ & $\tilde{O}\del{\frac{d^2}{(1-2\eta)^3}\frac1\epsilon}$\\
\bottomrule
\end{tabular}
\end{table}
\footnotetext{The algorithm needs to minimize 0-1 loss, the best known method for which requires superpolynomial time.}

\begin{table}
\centering
\caption{A comparison of algorithms for active learning of halfspaces under the uniform distribution, in the $\nu$-adversarial noise model.}
\label{tab:comp-an}
\begin{tabular}{llll}
\toprule
Algorithm & Noise Tolerance & Label Complexity & Time Complexity \\
\midrule
\cite{ZC14} & $\nu=\Omega(\epsilon)$  & $\tilde{O}( d \ln\frac1\epsilon )$ & $\superpoly(d,\frac1\epsilon)$\\
\cite{HKY15} & $\nu=\Omega(\epsilon)$ & $\tilde{O}( d \ln\frac1\epsilon )$ & $\poly(d,\frac1\epsilon)$ \\
Our Work & $\nu=\Omega(\frac \epsilon {\ln d + \ln\ln\frac1\epsilon})$ & $\tilde{O}(d \ln \frac1\epsilon) $ & $\tilde{O}\del{d^2 \cdot \frac1\epsilon}$\\
\bottomrule
\end{tabular}
\end{table}

%% file: relwork.tex
\section{Related Work}
\paragraph{Active Learning.} The recent decades have seen much success in both theory and practice of active learning; see the excellent surveys by~\cite{S10, H14, D11}. On the theory side, many label-efficient active learning algorithms have been proposed and analyzed. An incomplete list includes~\cite{CAL94, BBL09, H07, DHM07, H09, K10, ZC14, HAHLS15}. Most algorithms relies on solving empirical risk minimization problems, which are computationally hard in the presence of noise~\citep{ABSZ93}.

\paragraph{Computational Hardness of Learning Halfspaces.} Efficient learning of halfspaces is one of the central problems in machine learning~\cite{CS00}. In the realizable case, it is well known that linear programming will find a consistent hypothesis over data efficiently. In the nonrealizable setting, however, the problem is much more challenging.

A series of papers have shown the hardness of learning halfspaces with agnostic noise~\citep{ABSZ93, FGKP06, GR09, KK14, D15}. The state of the art result~\citep{D15} shows that under standard complexity-theoretic assumptions, there exists a data distribution, such that the best linear classifier has error $o(1)$, but no polynomial time algorithms can achieve an error at most $\frac{1}{2} - \frac{1}{d^c}$ for every $c>0$, even with improper learning. ~\cite{KK14} shows that under standard assumptions, %(learning $k$-sparse parity with noise must have time $n^{\Omega(k)}$)
even if the unlabeled distribution is Gaussian, any agnostic halfspace learning algorithm must run in time $(\frac 1 \epsilon)^{\Omega(\ln d)}$ to achieve an excess error of $\epsilon$. These results indicate that, to have nontrivial guarantees on learning halfspaces with noise in polynomial time, one has to make additional assumptions on the data distribution over instances and labels.

\paragraph{Efficient Active Learning of Halfspaces.} Despite considerable efforts, there are only a few halfspace learning algorithms that are both computationally-efficient and label-efficient even under the uniform distribution. In the realizable setting, \cite{DKM05,BBZ07, BL13} propose computationally efficient active learning algorithms which have an optimal label complexity of $\tilde{O}(d\ln \frac{1}{\epsilon})$.

%It leaves as an open problem whether perceptron-based algorithms can tolerate noise. \cite{ propose a margin-based algorithm. In the realizable setting, their algorithm can be implemented in polynomial time using linear programming, and achieves the optimal label complexity of $\tilde{O}(d\log \frac{1}{\epsilon})$.

Since it is believed to be hard for learning halfspaces in the general agnostic setting, it is natural to consider algorithms that work under more moderate noise conditions. Under the bounded noise setting~\cite{MN06}, the only known algorithms that are both label-efficient and computationally-efficient are \cite{ABHU15, ABHZ16}. \cite{ABHU15} uses a margin-based framework which queries the labels of examples near the decision boundary. To achieve computational efficiency, it adaptively chooses a sequence of hinge loss minimization problems to optimize as opposed to directly optimizing the 0-1 loss. It works only when the label flipping probability upper bound $\eta$ is small ($\eta \leq 1.8 \times 10^{-6}$). \cite{ABHZ16} improves over \cite{ABHU15} by adapting a polynomial regression procedure into the margin-based framework. It works for any $\eta < 1/2$, but its label complexity is $O(d^{O(\frac 1 {(1-2\eta)^4})} \ln\frac1\epsilon)$, which is far worse than the information-theoretic lower bound $\Omega(\frac{d}{(1-2\eta)^2}\ln\frac{1}{\epsilon})$.
Recently \cite{CHK17} gives an efficient algorithm with a near-optimal label complexity under the membership query model where the learner can query on synthesized points. In contrast, in our stream-based model, the learner can only query on points drawn from the data distribution. We note that learning in the stream-based model is harder than in the membership query model, and it is unclear how to transform the DC algorithm in \cite{CHK17} into a computationally efficient stream-based active learning algorithm.

Under the more challenging $\nu$-adversarial noise setting, \cite{ABL14} proposes a margin-based algorithm that reduces the problem to a sequence of hinge loss minimization problems. Their algorithm achieves an error of $\epsilon$ in polynomial time when $\nu = \Omega(\epsilon)$, but requires $\tilde{O}(d^2 \ln\frac1\epsilon)$ labels. Later, \cite{HKY15} performs a refined analysis to achieve a near-optimal label complexity of $\tilde{O}(d\ln \frac{1}{\epsilon})$, but the time complexity of the algorithm is still an unspecified high order polynomial.

Tables~\ref{tab:comp-bn} and~\ref{tab:comp-an} present comparisons between our results and
results most closely related to ours in the literature. Due to space limitations, discussions of additional related work are deferred to Appendix~\ref{sec:arw}.

%% file: setting.tex
\section{Definitions and Settings}
\label{sec:setting}

We consider learning homogeneous halfspaces under uniform distribution. The instance space $\calX$ is the unit sphere in $\R^d$, which we denote by $\S^{d-1}:=\left\{ x\in \R^d: \|x\| = 1 \right\}$. We assume $d \geq 3$ throughout this paper. The label space $\calY = \{+1, -1\}$. We assume all data points $(x,y)$ are drawn i.i.d. from an underlying distribution $D$ over $\calX \times \calY$. We denote by $D_\calX$ the marginal of $D$ over $\calX$ (which is uniform over $\S^{d-1}$), and $D_{Y \mid X}$ the conditional distribution of $Y$ given $X$.
Our algorithm is allowed to draw unlabeled examples $x \in \calX$ from $D_\calX$, and to make queries to a labeling oracle $\calO$ for labels. Upon query $x$, $\calO$ returns a label $y$ drawn from $D_{Y \mid X=x}$.
The hypothesis class of interest is the set of homogeneous halfspaces $\calH := \left\{h_w(x) = \sign(w \cdot x) \mid w \in \S^{d-1} \right\}$. For any hypothesis $h \in \calH$, we define its error rate $\err(h) := \P_D[h(X) \neq Y]$. We will drop the subscript $D$ in $\P_D$ when it is clear from the context. Given a dataset $S = \cbr{(X_1, Y_1), \ldots, (X_m, Y_m)}$, we define the empirical error rate of $h$ over $S$ as $\err_S(h) := \frac1m \sum_{i=1}^m \one\cbr{h(x_i) \neq y_i}$.

%Define $u$ be a unit vector such that $\err(h_u) = \min_{h \in \calH} \err(h)$. We define the Bayesian optimal classifier $h^*$ to be
%$$h^*(x)=\begin{cases}
%1 & \P(y=1 \mid x)>1/2 \\
%-1 & \P(y=1 \mid x)\leq1/2
%\end{cases}$$
%Note that $h^*$ may not be inside $\calH$.

%Thus, to find a classifier with low error, it suffices to find a classifier geometrically close to the optimal classifier.

\begin{definition}[Bounded Noise~\cite{MN06}]
We say that the labeling oracle $\calO$ satisfies the \emph{$\eta$-bounded noise condition} for some $\eta \in [0,1/2)$ with respect to $u$, if for any $x$, $\P[Y \neq \sign(u \cdot x)\mid X = x] \leq \eta$.
\end{definition}
% following two conditions hold:
%\begin{enumerate}
%\item $h^* \in \calH$;
%\item $\P(Y \neq h^*(X)\mid X) \leq \eta$.
%\end{enumerate}

It can be seen that under $\eta$-bounded noise condition, $h_u$ is the Bayes classifier.

\begin{definition}[Adversarial Noise~\cite{ABL14}]
We say that the labeling oracle $\calO$ satisfies the \emph{$\nu$-adversarial noise condition} for some $\nu \in [0,1]$  with respect to $u$, if $\P[Y \neq \sign(u \cdot X)] \leq \nu$.
\end{definition}

For two unit vectors $v_1, v_2$, denote by $\theta(v_1, v_2) = \arccos(v_1 \cdot v_2)$ the angle between them. The following lemma gives relationships between errors and angles (see also Lemma 1 in \cite{ABHZ16}).

%\cz{I do not see it used anywhere..}With some abuse of notations, we define $\theta(h_{v_1}, h_{v_2}) = \theta(v_1, v_2)$.

\begin{lemma}
For any $v_1,v_2\in\S^{d-1}$, $\left| \err(h_{v_1})-\err(h_{v_2})\right| \leq \P\left[h_{v_1}(X) \neq h_{v_2}(X)\right] = \frac{\theta(v_1, v_2)}{\pi}$.

Additionally, if the labeling oracle satisfies the $\eta$-bounded noise condition with respect to $u$, then for any vector $v$, $\left| \err(h_v)-\err(h_u)\right| \geq (1-2\eta) \P\left[h_v(X) \neq h_u(X)\right] = \frac{1-2\eta}{\pi}\theta(v, u)$.
\label{lem:triangle}
\end{lemma}

Given access to unlabeled examples drawn from $D_\calX$ and a labeling oracle $\calO$, our goal is to find a polynomial time algorithm $\calA$ such that with probability at least $1-\delta$, $\calA$ outputs a halfspace $h_{v} \in \calH$ with $\P[\sign(v \cdot X) \neq \sign(u \cdot X)] \leq \epsilon$ for some target accuracy $\epsilon$ and confidence $\delta$. (By Lemma~\ref{lem:triangle}, this guarantees that the excess error of $h_v$ is at most $\epsilon$, namely, $\err(h_v) - \err(h_u) \leq \epsilon$.) The desired algorithm should make as few queries to the labeling oracle $\calO$ as possible.

We say an algorithm $\calA$ achieves a \emph{label complexity} of $\Lambda(\epsilon, \delta)$, if for any target halfspace $h_u\in\calH$, with probability at least $1-\delta$, $\calA$ outputs a halfspace $h_{v} \in \calH$ such that $\err(h_v)\leq \err(h_u) + \epsilon$, and requests at most $\Lambda(\epsilon, \delta)$ labels from oracle $\calO$.

%% file: algorithm.tex
\section{Main Algorithm}

%\subsection{Techniques}
%\label{sec:techniques}

Our main algorithm, \AP (Algorithm~\ref{alg:activeperceptron}), works in epochs.
It works under the bounded and the adversarial noise
models, if its sample schedule $\cbr{m_k}$ and band width $\cbr{b_k}$ are set appropriately with respect to each noise model.
At the beginning of each epoch $k$,
it assumes an upper bound of $\frac{\pi}{2^k}$ on $\theta(v_{k-1}, u)$, the angle between
current iterate $v_{k-1}$ and the underlying halfspace $u$. As we will see, this can be shown to hold
 with high probability inductively. Then, it calls procedure
\MP (Algorithm~\ref{alg:modperceptron}) to find an new iterate
$v_k$, which can be shown to have an angle with $u$ at most $\frac{\pi}{2^{k+1}}$ with
high probability.
The algorithm ends when a total of $k_0 = \lceil \log_2\frac1\epsilon \rceil$
epochs have passed.
%The final iterate $v_{k_0}$ can be shown to have angle at most
%$\epsilon$ with $u$.

For simplicity, we assume for the rest of the paper that the angle between
the initial halfspace $v_0$ and the underlying halfspace $u$ is acute,
that is, $\theta(v_0, u) \leq \frac{\pi}{2}$; Appendix~\ref{sec:acute}
shows that this assumption can be removed with a
constant overhead in terms of label and time complexities.

\begin{algorithm}[H]
\caption{\AP}
\begin{algorithmic}[1]
\REQUIRE Labeling oracle $\calO$, initial halfspace $v_0$, target error $\epsilon$, confidence $\delta$, sample schedule $\cbr{m_k}$, band width $\cbr{b_k}$.
\ENSURE learned halfspace $v$.
\STATE Let $k_0 = \lceil \log_2 \frac{1}{\epsilon} \rceil$.
\FOR{$k = 1,2,\ldots,k_0$}
		\STATE $v_k \leftarrow \MP(\calO, v_{k-1}, \frac{\pi}{2^k}, \frac{\delta}{k(k+1)}, m_k, b_k)$.
\ENDFOR
\RETURN$v_{k_0}$.
\end{algorithmic}
\label{alg:activeperceptron}
\end{algorithm}

Procedure \MP (Algorithm~\ref{alg:modperceptron}) is the core component of \AP. It sequentially
performs a modified Perceptron
update rule on the selected new examples $(x_t, y_t)$~\citep{MS54, BFKV98, DKM05}:
\begin{equation}
  w_{t+1} \gets w_t - 2\one\cbr{y_t w_t \cdot x_t < 0} (w_t \cdot x_t) \cdot x_t
  \label{eqn:modperceptron-initial}
\end{equation}
%\cz{We note that this is also similar to Bilmes' observation that active learning can be viewed
%as non-convex optimization.}

Define $\theta_t := \theta(w_t, u)$. Update rule~\eqref{eqn:modperceptron-initial} implies the following relationship between $\theta_{t+1}$ and $\theta_t$
(See Lemma~\ref{lem:coschange} in Appendix~\ref{sec:progmeasure} for its proof):
\begin{equation}
  \cos\theta_{t+1} - \cos\theta_t = - 2\one\cbr{y_t w_t \cdot x_t < 0} (w_t \cdot x_t) \cdot (u \cdot x_t)
  \label{eqn:modperceptrontheta}
\end{equation}
This motivates us to take $\cos \theta_t$ as our measure of progress; we would like to drive $\cos\theta_t$ up to $1$(so that $\theta_t$ goes down to $0$) as fast as possible.

To this end, \MP samples new points $x_t$ under time-varying distributions $D_{\calX}|_{R_t}$ and query for their labels,
where $R_t = \cbr{x \in \S^{d-1}: \frac{b}{2} \leq w_t \cdot x \leq b }$ is a band inside
the unit sphere.
The rationale behind the choice of $R_t$ is twofold:
\begin{enumerate}
\item We set $R_t$ to have a probability mass of  $\tilde{\Omega}(\epsilon)$, so that the time complexity of rejection sampling is at most $\tilde{O}(\frac{1}{\epsilon})$ per example. Moreover, in the adversarial noise setting, we set $R_t$ large enough to dominate the noise of magnitude $\nu = \tilde{\Omega}(\epsilon)$.
\item Unlike the active Perceptron algorithm in~\cite{DKM05} or other margin-based approaches (for example~\cite{TK01, BBZ07}) where examples with small margin are queried, we query the label of the examples with a
range of margin $[\frac b 2, b]$. From a technical perspective, this ensures that $\theta_t$ decreases by a decent amount in expectation (see Lemmas~\ref{lem:delta-cos-geq-bn} and \ref{lem:delta-cos-geq-an} for details).
\end{enumerate}

Following the insight of~\cite{GCB09}, we remark that the modified Perceptron update~\eqref{eqn:modperceptron-initial} on distribution $D_{\calX}|_{R_t}$ can be alternatively viewed as performing stochastic gradient descent on a special non-convex loss function $\ell(w, (x, y)) = \min(1,\max(0, -1-\frac 2 b y w \cdot x ))$. It is an interesting open question whether optimizing this new loss function can lead to improved empirical results for learning halfspaces.

\begin{algorithm}[h]
\caption{\MP}
\begin{algorithmic}[1]
\REQUIRE Labeling oracle $\calO$, initial halfspace $w_0$, angle upper bound $\theta$, confidence $\delta$, number of iterations $m$, band width $b$.
\ENSURE Improved halfspace $w_m$.
\FOR{$t = 0,1,2,\ldots,m-1$}
  \STATE Define region $R_t = \cbr{x \in \S^{d-1}: \frac{b}{2} \leq w_t \cdot x \leq b }$.
	\STATE Rejection sample $x_t \sim D_\calX|_{R_t}$. In other words, draw $x_t$ from $D_\calX$ until
  $x_t$ is in $R_t$. Query $\calO$ for its label $y_t$.
	\STATE $w_{t+1} \leftarrow w_t - 2\one\cbr{y_t w_t \cdot x_t < 0} \cdot (w_t \cdot x_t) \cdot x_t$.
\ENDFOR
\RETURN $w_m$.
\end{algorithmic}
\label{alg:modperceptron}
\end{algorithm}

%% file: theory.tex
\section{Performance Guarantees}
We show that \AP works in the bounded and the adversarial noise
models, achieving computational efficiency and near-optimal label complexities. To this end, we first
give a lower bound on the label complexity under bounded noise, and then give computational and label complexity
upper bounds under the two noise conditions respectively. We defer all proofs to the Appendix.

\subsection{A Lower Bound under Bounded Noise}\label{subsec:lb}
We first present an information-theoretic lower bound on the label complexity in the bounded noise setting under uniform distribution. This extends the distribution-free lower bounds of~\cite{RR11,H14}, and
generalizes the realizable-case lower bound of~\cite{KMT93} to the bounded noise setting.
Our lower bound can also be viewed as an extension of ~\cite{WS16}'s Theorem 3; specifically it addresses the hardness under the $\alpha$-Tsybakov noise condition where $\alpha = 0$ (while~\cite{WS16}'s Theorem 3
provides lower boundes when $\alpha \in (0,1)$).

\begin{theorem}
For any $d>4$, $0\leq\eta<\frac{1}{2}$, $0<\epsilon\leq\frac{1}{4\pi}$,
$0<\delta\leq\frac{1}{4}$, for any active learning algorithm $\calA$, there is a
$u\in\S^{d-1}$, and a labeling oracle $\calO$ that satisfies $\eta$-bounded
noise condition with respect to $u$, such that if with probability at least $1-\delta$,
$\calA$ makes at most $n$ queries of labels to $\calO$ and outputs
$v\in\S^{d-1}$ such that $\P[\sign(v \cdot X) \neq \sign(u \cdot X)]\leq\epsilon$, then
$n\geq\Omega\left(\frac{d\log\frac{1}{\epsilon}}{(1-2\eta)^{2}}+\frac{\eta\log\frac{1}{\delta}}{(1-2\eta)^{2}}\right)$.
\label{thm:lb-bn}
\end{theorem}

\subsection{Bounded Noise}
We establish Theorem~\ref{thm:ap-bn} in the bounded noise setting. The theorem implies that, with
appropriate settings of input parameters, $\AP$ efficiently learns a halfspace of excess error at most $\epsilon$ with
probability at least $1-\delta$, under the assumption that $D_{\calX}$ is uniform over the unit sphere and $\calO$ has bounded noise. In addition, it queries at most $\tilde{O}(\frac{d}{(1-2\eta)^2} \ln\frac1\epsilon)$ labels. This matches the lower bound of Theorem~\ref{thm:lb-bn}, and improves over the state of the art result of~\cite{ABHZ16}, where
a label complexity of $\tilde{O}(d^{O(\frac 1 {(1-2\eta)^4})}\ln\frac1\epsilon)$ is shown using a different algorithm.

%This gives the first computationally efficient active learning algorithm with near-optimal
%$\tilde{O}(\frac{d}{(1-2\eta)^2} \ln\frac1\epsilon)$ label complexity (see the lower bound theorem above)
%
The proof and the precise setting of parameters ($m_k$ and $b_k$) are given in Appendix~\ref{sec:pf-main}.

\begin{theorem}[\AP under Bounded Noise]
Suppose Algorithm~\ref{alg:activeperceptron} has inputs labeling oracle $\calO$
that satisfies $\eta$-bounded noise condition with respect to halfspace $u$, initial halfspace $v_0$
such that $\theta(v_0, u) \in [0,\frac \pi 2]$,
target error $\epsilon$,
confidence $\delta$,
sample schedule $\cbr{m_k}$ where $m_k = \Theta\del{\frac{ d }{(1-2\eta)^{2}} (\ln\frac{ d }{(1-2\eta)^{2}} + \ln\frac{k}{\delta}) }$,
band width $\cbr{b_k}$ where $b_k = \Theta\del{\frac{ 2^{-k} (1-2\eta)}{\sqrt{d}\ln(km_k/\delta)}}$.
Then with probability at least $1-\delta$:
\begin{enumerate}[leftmargin=1cm]
\item The output halfspace $v$ is such that $\P[\sign(v \cdot X) \neq \sign(u \cdot X)] \leq \epsilon$.
\item The number of label queries is $O\del{\frac{d}{(1-2\eta)^2} \cdot \ln\frac{1}{\epsilon} \cdot \del{\ln\frac{d}{(1-2\eta)^2} + \ln\frac{1}{\delta} + \ln\ln\frac1\epsilon}}$.
\item The number of unlabeled examples drawn is \\$O\del{\frac{d}{(1-2\eta)^3} \cdot \del{\ln\frac{d}{(1-2\eta)^2} + \ln\frac1\delta + \ln\ln\frac{1}{\epsilon}}^2 \cdot \frac1\epsilon \ln\frac1\epsilon}$.
\item The algorithm runs in time $O\del{\frac{d^2}{(1-2\eta)^3} \cdot \del{\ln\frac{d}{(1-2\eta)^2} + \ln\frac1\delta + \ln\ln\frac{1}{\epsilon}}^2 \cdot \frac1\epsilon \ln\frac1\epsilon}$.
\end{enumerate}
\label{thm:ap-bn}
\end{theorem}

The theorem follows from Lemma~\ref{lem:mp-bn} below. The key ingredient of the
lemma is a delicate analysis of the dynamics of the angles $\cbr{\theta_t}_{t=0}^m$, where
$\theta_t = \theta(w_t, u)$ is the angle between the iterate $w_t$ and the halfspace $u$.
Since $x_t$ is randomly sampled and $y_t$ is noisy,
we are only able to show that $\theta_t$
decreases by a decent amount {\em in expectation}. To remedy the stochastic fluctuations,
we apply martingale concentration inequalities to carefully control the upper envelope of
sequence $\cbr{\theta_t}_{t=0}^m$.
%Moreover the setting of $b$
%and $\theta_t$ need to be compatible for th

\begin{lemma}[\MP under Bounded Noise]
Suppose Algorithm~\ref{alg:modperceptron} has inputs labeling oracle $\calO$ that satisfies $\eta$-bounded noise condition with respect to halfspace $u$,
initial halfspace $w_0$ and
angle upper bound $\theta \in (0, \frac \pi 2]$ such that $\theta(w_0, u) \leq \theta$,
confidence $\delta$,
number of iterations
$m = \Theta(\frac{ d }{(1-2\eta)^{2}} (\ln\frac{ d }{(1-2\eta)^{2}} + \ln\frac{1}{\delta}))$,
band width $b = \Theta\del{\frac{ \theta (1-2\eta)}{\sqrt{d}\ln(m/\delta)}}$.
Then with probability at least $1-\delta$:
\begin{enumerate}[leftmargin=1cm]
  \item The output halfspace $w_m$ is such that $\theta(w_m, u) \leq \frac{\theta}{2}$.
  \item The number of label queries is $O\del{\frac{d}{(1-2\eta)^2} \del{\ln\frac{d}{(1-2\eta)^2} + \ln\frac1\delta} }$.
  \item The number of unlabeled examples drawn is $O\del{\frac{d}{(1-2\eta)^3} \cdot \del{\ln\frac{d}{(1-2\eta)^2} + \ln\frac1\delta}^2 \cdot \frac1\theta}$.
  \item The algorithm runs in time $O\del{\frac{d^2}{(1-2\eta)^3} \cdot \del{\ln\frac{d}{(1-2\eta)^2} + \ln\frac1\delta}^2 \cdot \frac1\theta}$.
\end{enumerate}
\label{lem:mp-bn}
\end{lemma}

\subsection{Adversarial Noise}
We establish Theorem~\ref{thm:ap-an} in the adversarial noise setting. The theorem implies that, with
appropriate settings of input parameters, $\AP$ efficiently learns a halfspace of excess error at most $\epsilon$ with
probability at least $1-\delta$, under the assumption that $D_{\calX}$ is uniform over the unit sphere and $\calO$ has an adversarial noise of magnitude $\nu = \Omega(\frac{\epsilon}{\ln d + \ln\ln\frac1\epsilon})$. In addition, it queries at most $\tilde{O}(d \ln\frac1\epsilon)$ labels. Our label complexity bound is information-theoretically optimal~\cite{KMT93}, and matches the state of the art result of~\cite{HKY15}.
The benefit of our approach is computational: it has a running time of $\tilde O(\frac {d^2} \epsilon)$, while~\cite{HKY15} needs to solve a convex optimization problem whose running time is some polynomial over $d$ and $\frac 1 \epsilon$ with an unspecified degree.

%In the adversarial noise case, we show that \AP can tolerate a noise level of
%. This is slightly weaker
%than $\Omega(\epsilon)$ shown in~\cite{ABL14}.
%However, we show that \AP has
%a near-optimal label complexity of $\tilde{O}(d\ln\frac1\epsilon)$, which improves over
%the result of~\cite{ABL14} by a factor of $d$.
%This shows the flexibility of the algorithmic framework
%of \AP.

The proof and the precise setting of parameters ($m_k$ and $b_k$) are given in Appendix~\ref{sec:pf-main}.

\begin{theorem}[\AP under Adversarial Noise]
Suppose Algorithm~\ref{alg:activeperceptron} has inputs labeling oracle $\calO$
that satisfies $\nu$-adversarial noise condition with respect to halfspace $u$, initial halfspace $v_0$
such that $\theta(v_0, u) \leq \frac{\pi}{2}$,
target error $\epsilon$,
confidence $\delta$,
sample schedule $\cbr{m_k}$ where $m_k = \Theta(d (\ln d + \ln\frac{k}{\delta}))$,
band width $\cbr{b_k}$ where $b_k = \Theta\del{\frac{ 2^{-k} }{\sqrt{d}\ln(km_k/\delta)}}$.
Additionally $\nu \leq \Omega(\frac{\epsilon}{\ln \frac d \delta  + \ln\ln\frac1\epsilon})$.
Then with probability at least $1-\delta$:
\begin{enumerate}[leftmargin=1cm]
\item The output halfspace $v$ is such that $\P[\sign(v \cdot X) \neq \sign(u \cdot X)] \leq \epsilon$.
\item The number of label queries is $O\del{d \cdot \ln\frac{1}{\epsilon} \cdot \del{\ln d + \ln\frac{1}{\delta} + \ln\ln\frac1\epsilon}}$.
\item The number of unlabeled examples drawn is $O\del{d \cdot \del{\ln d + \ln\frac1\delta + \ln\ln\frac{1}{\epsilon}}^2 \cdot \frac1\epsilon \ln\frac1\epsilon}$.
\item The algorithm runs in time $O\del{d^2 \cdot \del{\ln d + \ln\frac1\delta + \ln\ln\frac{1}{\epsilon}}^2 \cdot \frac1\epsilon \ln\frac1\epsilon}$.
\end{enumerate}
\label{thm:ap-an}
\end{theorem}

The theorem follows from Lemma~\ref{lem:mp-an} below, whose proof is similar
to Lemma~\ref{lem:mp-bn}.

\begin{lemma}[\MP under Adversarial Noise]
Suppose Algorithm~\ref{alg:modperceptron} has inputs labeling oracle $\calO$
that satisfies $\nu$-adversarial noise condition with respect to halfspace $u$,
initial halfspace $w_0$ and
angle upper bound $\theta \in (0, \frac \pi 2]$ such that $\theta(w_0, u) \leq \theta$, confidence $\delta$,
number of iterations
$m = \Theta(d(\ln d + \ln\frac{1}{\delta}))$,
band width $b = \Theta\del{\frac{ \theta }{\sqrt{d}\ln(m/\delta)}}$.
Additionally $\nu \leq \Omega(\frac{\theta}{\ln (m/\delta))})$.
Then with probability at least $1-\delta$:
\begin{enumerate}[leftmargin=1cm]
  \item The output halfspace $w_m$ is such that $\theta(w_m, u) \leq \frac{\theta}{2}$.
  \item The number of label queries is $O\del{d \cdot \del{\ln d + \ln\frac{1}{\delta} }}$.
  \item The number of unlabeled examples drawn is $O\del{d \cdot \del{\ln d + \ln\frac1\delta }^2 \cdot \frac1\theta }$
  \item The algorithm runs in time $O\del{d^2 \cdot \del{\ln d + \ln\frac1\delta }^2 \cdot \frac1\theta }$.
\end{enumerate}
\label{lem:mp-an}
\end{lemma}

%% file: passive.tex
\section{Implications to Passive Learning}
\label{sec:passive}
$\AP$ can be converted to a passive learning algorithm, $\PP$, for learning homogeneous halfspaces under the uniform distribution over the unit sphere. $\PP$ has PAC sample complexities close to the lower bounds under the two noise models.
We give a formal description of $\PP$ in Appendix~\ref{sec:passive-alg}. We give its formal guarantees in the
corollaries below, which are immediate consequences of Theorems~\ref{thm:ap-bn} and~\ref{thm:ap-an}.

In the $\eta$-bounded noise model, the sample complexity of \PP improves over the state of the art result of~\cite{ABHZ16}, where a sample complexity of $\tilde{O}(\frac {d^{O(\frac{1}{(1-2\eta)^4})}} \epsilon)$ is obtained. The bound has the same dependency on $\epsilon$ and $d$ as the minimax upper bound of $\tilde \Theta(\frac{d}{\epsilon (1-2\eta)})$ by~\cite{MN06}, which is achieved by a computationally inefficient ERM algorithm.

\begin{corollary}[\PP under Bounded Noise]
Suppose \PP has inputs distribution $D$ that satisfies $\eta$-bounded noise condition with respect to $u$,
initial halfspace $v_0$,
target error $\epsilon$,
confidence $\delta$,
sample schedule $\cbr{m_k}$ where $m_k = \Theta\del{\frac{ d }{(1-2\eta)^{2}} (\ln\frac{ d }{(1-2\eta)^{2}} + \ln\frac{k}{\delta}) }$,
band width $\cbr{b_k}$ where $b_k = \Theta\del{\frac{ 2^{-k} (1-2\eta)}{\sqrt{d} \ln (km_k /\delta)}}$.
Then with probability at least $1-\delta$:
(1) The output halfspace $v$ is such that $\err(h_{v}) \leq \err(h_u)+\epsilon$;
(2) The number of labeled examples drawn is $\tilde{O}\del{\frac{d}{(1-2\eta)^3\epsilon} }$.
(3) The algorithm runs in time $\tilde{O}\del{\frac{d^2}{(1-2\eta)^3\epsilon} }$.
\label{cor:pp-bn}
\end{corollary}

In the $\nu$-adversarial noise model, the sample complexity of \PP matches the minimax optimal sample complexity upper bound of $\tilde{\Theta}(\frac {d} \epsilon)$ obtained in ~\cite{HKY15}.
%\footnote{Note that the information-theoretic lower bound in this setting is at least $\Omega(\frac d \epsilon)$, which is the lower bound in the realizable case~\cite{L95}.}
Same as in active learning, our algorithm has a faster running time than~\cite{HKY15}.

\begin{corollary}[\PP under Adversarial Noise]
Suppose \PP has inputs distribution $D$ that satisfies $\nu$-adversarial noise condition with respect to $u$,
initial halfspace $v_0$,
target error $\epsilon$,
confidence $\delta$,
sample schedule $\cbr{m_k}$ where $m_k = \Theta\del{d (\ln d + \ln\frac{k}{\delta}) }$,
band width $\cbr{b_k}$ where $b_k = \Theta\del{\frac{ 2^{-k} }{\sqrt{d} \ln (km_k/\delta)}}$.
Furthermore $\nu = \Omega(\frac{\epsilon}{\ln\ln\frac1 \epsilon + \ln\frac d \delta})$.
Then with probability at least $1-\delta$:
(1) The output halfspace $v$ is such that $\err(h_{v}) \leq \err(h_u)+\epsilon$;
(2) The number of labeled examples drawn is $\tilde{O}\del{ \frac d \epsilon}$.
(3) The algorithm runs in time $\tilde{O}\del{\frac {d^2} \epsilon}$.
\label{cor:pp-an}
\end{corollary}

Tables~\ref{tab:pl-comp-bn} and~\ref{tab:pl-comp-an} present comparisons between our results and results most closely related to ours.

%$\tilde O(\frac {\exp(O( \frac{\log d}{(1-2\eta)^4}))} \epsilon)$
%\cite{ABHZ16} & $\tilde O \del{ \frac {d^{O(\frac{1}{(1-2\eta)^4})}} \epsilon}$ & $\tilde O\del{\frac {d^{O(\frac{1}{(1-2\eta)^4})}} \epsilon}$\\
\begin{table}
\caption{A comparison of algorithms for PAC learning halfspaces under the uniform distribution, in the $\eta$-bounded noise model.}
\label{tab:pl-comp-bn}
\centering
\begin{tabular}{lll}
\toprule
Algorithm & Sample Complexity & Time Complexity \\
\midrule
\cite{ABHZ16} & $\tilde O(\frac {d^{O(\frac{1}{(1-2\eta)^4})}} \epsilon)$ & $\tilde O(\frac {d^{O(\frac{1}{(1-2\eta)^4})}} \epsilon)$\\

ERM~\citep{MN06} &  $\tilde{O}(\frac{d}{(1-2\eta) \epsilon})$ & $\superpoly(d,\frac1\epsilon)$ \\

Our Work & $\tilde{O}(\frac{d}{(1-2\eta)^3 \epsilon})$ & $\tilde{O}(\frac{d^2}{(1-2\eta)^3} \cdot \frac1\epsilon)$\\

%Lower Bound~\citep{MN06} & $\tilde{\Omega}(\frac{d}{(1-2\eta) \epsilon})$ & - \\
\bottomrule
\end{tabular}
\end{table}

\begin{table}
\caption{A comparison of algorithms for PAC learning halfspaces under the uniform distribution, in the $\nu$-adversarial noise model where $\nu = \Omega(\frac{\epsilon}{\ln\ln\frac 1 \epsilon + \ln d})$.}
\label{tab:pl-comp-an}
\centering
\begin{tabular}{lll}
\toprule
Algorithm & Sample Complexity & Time Complexity \\
\midrule
\cite{HKY15} & $\tilde O(\frac {d} \epsilon)$ & $\poly(d, \frac 1 \epsilon)$ \\

ERM~\citep{VC71} &  $\tilde{O}(\frac d \epsilon)$ & $\superpoly(d,\frac1\epsilon)$ \\

Our Work & $\tilde{O}(\frac d \epsilon)$ & $\tilde{O}(\frac {d^2} \epsilon)$\\

%Lower Bound~\citep{L95} & $\tilde{\Omega}(\frac{d}{\epsilon})$ & - \\
\bottomrule
\end{tabular}
\end{table}

%% file: appendix.tex
\appendix
\section{Additional Related Work}
\label{sec:arw}
\paragraph{Active Learning.} The recent decades have seen much success in both theory and practice of active learning; see the excellent surveys by~\cite{S10, H14, D11}. On the theory side, many label-efficient active learning algorithms have been proposed and analyzed~\cite{CAL94, FSST97, D05, BBL09, H07, BBZ07, DHM07, BHV10, BDL09, H09, K10, H10, BHLZ10, W11, H11, ABE14, ZC14, HAHLS15}. Most algorithms are {\em disagreement-based} algorithms~\citep{H14}, and are not label-optimal due to the conservativeness of their label query policy.
In addition, most of these algorithms require either explicit enumeration of classifiers in the hypothesis classes, or solving empirical 0-1 loss minimization problems on sets of examples. The former approach is easily seen to be computationally
infeasible, while the latter is proven to be computationally hard as well~\cite{ABSZ93}.
The only exception in this family we are aware of is~\cite{HY12}. \cite{HY12} considers active learning by sequential convex surrogate loss minimization. However, it assumes that the expected convex loss minimizer over all possible functions lies in a pre-specified real-valued function class, which is unlikely to hold in the bounded noise and the adversarial noise settings.

Some recent works~\cite{ZC14, HAHLS15, BBZ07, BL13, WS16} provide noise-tolerant active learning algorithms with improved label complexity over disagreement-based approaches. However, they are still computationally inefficient:~\cite{ZC14} relies on solving a series of linear program with an exponential number of constraints, which are computationally intractable; ~\cite{HAHLS15, BBZ07, BL13, WS16} relies on solving a series of empirical 0-1 loss minimization problems, which are also computationally hard in the presence of noise~\cite{ABSZ93}.  %\cz{Seems like we should also mention MBAL if written this way.}
%For example, in the bounded noise setting, the hinge loss minimizer is
% $\sign(u \cdot x)$, which is not a linear function of $x$.
%Another crucial drawback is that, most of these algorithms need to solve empirical risk minimization problems, which is computationally hard in the presence of noise even for learning halfspaces under uniform distribution~\citep{KK14}.
% Recently~\cite{ZC14} proposes confidence-based active learning, where they can handle general data distribution, achieve statistical consistency, and have good label complexity bounds. Unfortunately, the algorithms are still intractable.
\paragraph{Efficient Learning of Halfspaces.} A series of papers have shown the hardness of learning halfspaces with agnostic noise~\citep{ABSZ93, FGKP06, GR09, KK14, D15}. These results indicate that, to have nontrivial guarantees on learning halfspaces with noise in polynomial time, one has to make additional assumptions on the data distribution over instances and labels.

Many noise models, other than the bounded noise model and the adversarial noise model, has been studied in the literature. A line of work~\cite{CGO09,OC11,DGS12,A13} considers parameterized noise models. For instance, \cite{DGS12} gives an efficient algorithm for the setting that $\E[Y| X=x]= u \cdot x$ where $u$ is the optimal classifier. \cite{A13} studies a generalization of the above linear noise model, where $Y$ is a multiclass label, and there is a link function $\Phi$ such that $\E[Y|X=x] = \nabla \Phi(u \cdot x)$. Their analyses depend heavily on the noise models and it is unknown whether their algorithms can work with more general noise settings.
\cite{ZLC17} analyzes the problem of learning halfspaces under a new noise condition (as an application of their general analysis of stochastic gradient Langevin dynamics). They assume that the label flipping probability on every $x$ is bounded by $\frac 1 2 - c |u \cdot x|$, for some $c \in (0,\frac 1 2]$. It can be seen that the bounded noise condition implies the noise condition of~\cite{ZLC17}, and it is an interesting open question whether it is possible to extend our algorithm and analysis to their setting.

%\cz{Cite Yuchen Zhang's result here as a noise model}

Under the random classification noise condition~\cite{AL88}, ~\cite{BFKV98} gives the first efficient passive learning algorithm of learning halfspaces, by using a modification of Perceptron update (similar to Equation~\eqref{eqn:modperceptron-initial}) together with a boosting-type aggregation. \cite{BF13} proposes an active statistical query algorithm for learning halfspaces. The algorithm proceeds by estimating the distance between the current halfspace and the optimal halfspace. However, it requires a suboptimal number of $\tilde{O}(\frac{d^2}{(1-2\eta)^2})$ labels. In addition, both results above rely on the uniformity over the random classification noise, and it is shown in~\cite{ABHU15} that this type of statistical query algorithms will fail in the heterogeneous noise setting (in particular the bounded noise setting and the adversarial noise setting).

In the adversarial noise model, we assume that there is a halfspace $u$ with error at most $\nu$ over data.
The goal is to design an efficient algorithm that outputting a classifier that disagrees with $u$ with probability at most $\epsilon$.
%This setting has been studied in~\cite{KKMS08, KLS09}.
\cite{KKMS08} proposes an elegant averaging-based
algorithm that tolerates an error of at most $\nu = \Omega(\frac{\epsilon}{\ln\frac1\epsilon})$ assuming that the
unlabeled distribution is uniform. However it has a suboptimal
label complexity of $\tilde O(\frac {d^2} {\epsilon^2})$.
Under the assumption that the unlabeled distribution is log-concave or $s$-concave, the state of the art results~\cite{ABL14, BZ17} give efficient margin-based algorithms that tolerates a noise of $\nu = \tilde\Omega(\epsilon)$.
As discussed in the main text, such algorithms require a hinge loss minimization procedure that has a running time polynomial in $d$ with an unspecified degree.
Finally, ~\cite{D15} gives a PTAS that outputs a classifier with error $(1+\mu)\nu + \epsilon$, in time
$O(\poly( d^{\tilde{O}(\frac{1}{\mu^2})}, \frac 1 \epsilon) )$. Observe that in the case of $\nu = O(\epsilon)$, the
running time is an unspecified high order polynomial in terms of $d$ and $\frac 1 \epsilon$.

%\cz{Daniel Kane's result on broader concept classes with stronger noise?}

%\cz{At some point we should mention, HY12 does not work, this has also been pointed out by ABHU15}

%\cz{The work by Ben-David et al on minimizing surrogate loss -> minimizing 0/1 loss is also relevant}

%In the distribution-free case,~\cite{BS12} show a kernel-based algorithm that has error $(1+\mu)\err_\gamma(u) + \epsilon$, running in time $poly(\frac{1}{\gamma}, \exp(\frac{1}{\gamma^2\mu^2}))$, where $\err_\gamma(w)$ is the $\gamma$-margin error of the hypothesis $w$. ~\cite{SSSS10} show a similar algorithm that has additive error $\epsilon$ that runs in time $poly(\exp(\frac{1}{\gamma} \log\frac{L}{\epsilon}))$.

\section{Implications to Passive Learning}
\label{sec:passive-alg}
In this section, we formally describe $\PP$ (Algorithm~\ref{alg:passiveperceptron}),
a passive learning version of Algorithm~\ref{alg:activeperceptron}. The algorithmic
framework is similar to Algorithm~\ref{alg:activeperceptron}, except that
it calls Algorithm~\ref{alg:passivemodperceptron} rather than
Algorithm~\ref{alg:modperceptron}.

\begin{algorithm}[H]
\caption{\PP}
\begin{algorithmic}[1]
\REQUIRE Initial halfspace $v_0$, target error $\epsilon$, confidence $\delta$, sample schedule $\cbr{m_k}$, band width $\cbr{b_k}$.
\ENSURE learned halfspace $\hat{v}$.
\STATE Let $k_0 = \lceil \log_2 \frac{1}{\epsilon} \rceil$.
\FOR{$k = 1,2,\ldots,k_0$}
		\STATE $v_k \leftarrow \PMP(\calO, v_{k-1}, \frac{\pi}{2^k}, \frac{\delta}{k(k+1)}, m_k, b_k)$.
\ENDFOR
\RETURN$v_{k_0}$.
\end{algorithmic}
\label{alg:passiveperceptron}
\end{algorithm}

Algorithm~\ref{alg:passivemodperceptron} is similar to
Algorithm~\ref{alg:modperceptron}, except that it
draws labeled examples from $D$ directly, as opposed to performing label queries
on unlabeled examples drawn.

\begin{algorithm}[h]
\caption{\PMP}
\begin{algorithmic}[1]
\REQUIRE Initial halfspace $w_0$, angle upper bound $\theta$, confidence $\delta$, number of iterations $m$, band width $b$.
\ENSURE Improved halfspace $w_m$.
\FOR{$t = 0,1,2,\ldots,m-1$}
  \STATE Define region $C_t = \cbr{(x,y) \in \S^{d-1} \times \cbr{-1,+1}: \frac{b}{2} \leq w_t \cdot x \leq b }$.
	\STATE Rejection sample $(x_t,y_t) \sim D|_{C_t}$. In other words, repeat drawing example $(x_t, y_t) \sim D$ until it is in $C_t$.
	\STATE $w_{t+1} \leftarrow w_t - 2\one\cbr{y_t w_t \cdot x_t < 0} \cdot (w_t \cdot x_t) \cdot x_t$.
\ENDFOR
\RETURN $w_m$.
\end{algorithmic}
\label{alg:passivemodperceptron}
\end{algorithm}

It can be seen that with the same input as $\AP$, $\PP$ has exactly the same running
time, and the number of labeled examples drawn in $\PP$ is exactly the same as the number of
unlabeled examples drawn in $\AP$. Therefore, Corollaries~\ref{cor:pp-bn} and~\ref{cor:pp-an} are immediate consequences of Theorems~\ref{thm:ap-bn} and~\ref{thm:ap-an}.

\section{Proofs of Theorems~\ref{thm:ap-bn} and~\ref{thm:ap-an}} \label{sec:pf-main}

In this section, we give straightforward proofs that show Theorem~\ref{thm:ap-bn}
(resp. Theorem~\ref{thm:ap-an}) are direct consequences of Lemma~\ref{lem:mp-bn}
(resp. Lemma~\ref{lem:mp-an}). We defer the proofs of Lemmas~\ref{lem:mp-bn} and~\ref{lem:mp-an}
to Appendix~\ref{sec:pf-mp}.

\begin{theorem}[Theorem~\ref{thm:ap-bn} Restated]
Suppose Algorithm~\ref{alg:activeperceptron} has inputs labeling oracle $\calO$
that satisfies $\eta$-bounded noise condition with respect to underlying halfspace $u$, initial halfspace $v_0$
such that $\theta(v_0, u) \leq \frac{\pi}{2}$,
target error $\epsilon$,
confidence $\delta$,
sample schedule $\cbr{m_k}$ where
$m_k = \lceil \frac{(3200\pi)^3 d }{(1-2\eta)^{2}} (\ln\frac{(3200\pi)^3 d }{(1-2\eta)^{2}} + \ln\frac{k(k+1)}{\delta}) \rceil$,
band width $\cbr{b_k}$ where
$b_k = \frac{1}{2(600\pi)^2\ln\frac{m_k^2 k(k+1)}{\delta}} \frac{ 2^{-k} \pi (1-2\eta)}{\sqrt{d}}$.
Then with probability at least  $1-\delta$:
\begin{enumerate}
\item The output halfspace $v$ is such that $\P[\sign(v \cdot X) \neq \sign(u \cdot X)] \leq \epsilon$.
\item The number of label queries is $O\del{\frac{d}{(1-2\eta)^2} \cdot \ln\frac{1}{\epsilon} \cdot \del{\ln\frac{d}{(1-2\eta)^2} + \ln\frac{1}{\delta} + \ln\ln\frac1\epsilon}}$.
\item The number of unlabeled examples drawn is $O\del{\frac{d}{(1-2\eta)^3} \cdot \del{\ln\frac{d}{(1-2\eta)^2} + \ln\frac1\delta + \ln\ln\frac{1}{\epsilon}}^2 \cdot \frac1\epsilon \ln\frac1\epsilon}$.
\item The algorithm runs in time $O\del{\frac{d^2}{(1-2\eta)^3} \cdot \del{\ln\frac{d}{(1-2\eta)^2} + \ln\frac1\delta + \ln\ln\frac{1}{\epsilon}}^2 \cdot \frac1\epsilon \ln\frac1\epsilon}$.
\end{enumerate}
\label{thm:ap-bn-restated}
\end{theorem}

\begin{proof}[Proof of Theorem~\ref{thm:ap-bn-restated}]
From Lemma~\ref{lem:mp-bn}, we know that for every $k$, there is an event $E_k$
such that $\P(E_k) \geq 1 - \frac{\delta}{k(k+1)}$, and on event $E_k$,
items 1 to 4 of Lemma~\ref{lem:mp-bn} hold for input $w_0 = v_{k-1}$,
output $w_m = v_k$, $\theta = \frac{\pi}{2^k}$, $\delta = \frac{\delta}{k(k+1)}$.

Define event $E = \cup_{k=1}^{k_0} E_k$. By union bound, $\P(E) \geq 1-\delta$.
We henceforth condition on event $E$ happening.
\begin{enumerate}
\item By induction, the final output $v = v_{k_0}$ is such that
      $\theta(v, u) \leq 2^{-k_0} \pi \leq \epsilon \pi$, implying
			that $\P[\sign(v \cdot X) \neq \sign(u \cdot X)] \leq \epsilon$.
\item Define the number of label queries to oracle $\calO$ at iteration $k$ as $m_k$.
      On event $E_k$, $m_k$ is at most
			$O \del{\frac{d}{(1-2\eta)^2} \del{\ln\frac{d}{(1-2\eta)^2} + \ln\frac{k}\delta}}$.
			Thus, the total number of label queries to oracle $\calO$ is $\sum_{k=1}^{k_0} m_k$, which is
			at most
			\[ k_0 \cdot m_{k_0} = O \del{ k_0 \cdot \frac{d}{(1-2\eta)^2} \del{\ln\frac{d}{(1-2\eta)^2} + \ln\frac{k_0}\delta} }. \]
			Item 2 is proved by noting that $k_0 \leq \log\frac1\epsilon + 1$.
\item Define the number of unlabeled examples drawn iteration $k$ as $n_k$.
      On event $E_k$, $n_k$ is at most
			$O\del{\frac{d}{(1-2\eta)^3} \cdot \del{\ln\frac{d}{(1-2\eta)^2} + \ln\frac{k}\delta}^2 \cdot \frac1\epsilon}$.
			Thus, the total number of unlabeled examples drawn is $\sum_{k=1}^{k_0} n_k$, which is
			at most
			\[ k_0 n_{k_0} = O\del{k_0 \cdot \frac{d}{(1-2\eta)^3} \cdot \del{\ln\frac{d}{(1-2\eta)^2} + \ln\frac{k_0}\delta}^2 \cdot \frac1\epsilon}. \]
			Item 3 is proved by noting that $k_0 \leq \log\frac1\epsilon + 1$.
\item Item 4 is immediate from Item 3 and the fact that the time for processing each example is at most $O(d)$.
\qedhere
\end{enumerate}
\end{proof}

\begin{theorem}[Theorem~\ref{thm:ap-an} Restated]
Suppose Algorithm~\ref{alg:activeperceptron} has inputs labeling oracle $\calO$
that satisfies $\nu$-adversarial noise condition with respect to underlying halfspace $u$, initial halfspace $v_0$
such that $\theta(v_0, u) \leq \frac{\pi}{2}$,
target error $\epsilon$,
confidence $\delta$,
sample schedule $\cbr{m_k}$ where
$m_k = \lceil (3200\pi)^3 d (\ln(3200\pi)^3 d + \ln\frac{k(k+1)}{\delta}) \rceil$,
band width $\cbr{b_k}$ where
$b_k = \frac{1}{2(600\pi)^2\ln\frac{m_k^2 k(k+1)}{\delta}} \frac{ 2^{-k} \pi }{\sqrt{d}}$.
Additionally $\nu \leq \frac{\epsilon}{384(600\pi)^4 (4 \ln((3200\pi)^3 d) + 8 \ln\ln\frac1\epsilon + \ln\frac1\delta) }$.
Then with probability at least  $1-\delta$:
\begin{enumerate}
\item The output halfspace $v$ is such that $\P[\sign(v \cdot X) \neq \sign(u \cdot X)] \leq \epsilon$.
\item The number of label queries is $O\del{d \cdot \ln\frac{1}{\epsilon} \cdot \del{\ln d + \ln\frac{1}{\delta} + \ln\ln\frac1\epsilon}}$.
\item The number of unlabeled examples drawn is $O\del{d \cdot \del{\ln d + \ln\frac1\delta + \ln\ln\frac{1}{\epsilon}}^2 \cdot \frac1\epsilon \ln\frac1\epsilon}$.
\item The algorithm runs in time $O\del{d^2 \cdot \del{\ln d + \ln\frac1\delta + \ln\ln\frac{1}{\epsilon}}^2 \cdot \frac1\epsilon \ln\frac1\epsilon}$.
\end{enumerate}
\label{thm:ap-an-restated}
\end{theorem}

\begin{proof}[Proof of Theorem~\ref{thm:ap-an-restated}]
From Lemma~\ref{lem:mp-an}, we know that for every $k$, there is an event $E_k$
such that $\P(E_k) \geq 1 - \frac{\delta}{k(k+1)}$, and on event $E_k$,
items 1 to 4 of Lemma~\ref{lem:mp-an} hold for input $w_0 = v_k$,
output $w_m = v_{k+1}$, $\theta = \frac{\pi}{2^k}$.

Define event $E = \cup_{k=1}^{k_0} E_k$. By union bound, $\P(E) \geq 1-\delta$.
We henceforth condition on event $E$ happening.
\begin{enumerate}
\item By induction, the final output $v = v_{k_0}$ is such that that
      $\theta(v, u) \leq 2^{-k_0} \pi \leq \epsilon \pi$, implying
			that $\P[\sign(v \cdot X) \neq \sign(u \cdot X)] \leq \epsilon$.
\item Define the number of label queries to oracle $\calO$ at iteration $k$ as $m_k$.
      On event $E_k$, $m_k$ is at most
			$O \del{d \del{\ln d + \ln\frac{k}\delta}}$.
			Thus, the total number of label queries to oracle $\calO$ is $\sum_{k=1}^{k_0} m_k$, which is
			at most
			\[ k_0 \cdot m_{k_0} = O \del{ k_0 \cdot d \del{\ln d + \ln\frac{k_0}\delta} }. \]
			Item 2 is proved by noting that $k_0 \leq \log\frac1\epsilon + 1$.
\item Define the number of unlabeled examples drawn iteration $k$ as $n_k$.
      On event $E_k$, $n_k$ is at most
			$O\del{ d \cdot \del{\ln d + \ln\frac{k}\delta}^2 \cdot \frac1\epsilon}$.
			Thus, the total number of unlabeled examples drawn is $\sum_{k=1}^{k_0} n_k$, which is
			at most
			\[ k_0 n_{k_0} =  O\del{k_0 \cdot d \cdot \del{\ln d + \ln\frac{k_0}\delta}^2 \cdot \frac1\epsilon}. \]
			Item 3 is proved by noting that $k_0 \leq \log\frac1\epsilon + 1$.
\item Item 4 is immediate from Item 3 and the fact that the time for processing each example is at most $O(d)$. \qedhere
\end{enumerate}
\end{proof}

\section{Performance Guarantees of \MP} \label{sec:pf-mp}
In this section, we prove Lemmas~\ref{lem:mp-bn} and~\ref{lem:mp-an}, which guarantees
the shrinkage of $\theta_t$. Two major building blocks of
Lemma~\ref{lem:mp-bn} (resp. Lemma~\ref{lem:mp-an}) are Lemmas~\ref{lem:angle-shrinkage}
and~\ref{lem:delta-cos-geq-bn} (resp. Lemmas~\ref{lem:angle-shrinkage}
and~\ref{lem:delta-cos-geq-an}).
In essence, Lemma~\ref{lem:angle-shrinkage} turns per-iteration in-expectation guarantees provided by Lemmas~\ref{lem:delta-cos-geq-bn} and~\ref{lem:delta-cos-geq-an} into high probability upper bounds on the final $\theta_m$. We present Lemma~\ref{lem:angle-shrinkage} and its proof in detail in this section, and defer Lemmas~\ref{lem:delta-cos-geq-bn} and~\ref{lem:delta-cos-geq-an} to Appendix~\ref{sec:progmeasure}.

\begin{lemma}[Lemma~\ref{lem:mp-bn} Restated]
Suppose Algorithm~\ref{alg:modperceptron} has inputs labeling oracle $\calO$ that satisfies $\eta$-bounded noise condition with respect to underlying halfspace $u$,
initial vector $w_0$ and
angle upper bound $\theta \in (0, \frac \pi 2)$ such that $\theta(w_0, u) \leq \theta$, confidence $\delta$,
number of iterations $m = \lceil \frac{(3200\pi)^3 d }{(1-2\eta)^{2}} (\ln\frac{(3200\pi)^3 d }{(1-2\eta)^{2}} + \ln\frac{1}{\delta}) \rceil$,
band width
$b = \frac{1}{2(600\pi)^2\ln\frac{m^2}{\delta}} \frac{\theta (1-2\eta)}{\sqrt{d}}$.
then with probability at least $1-\delta$:
\begin{enumerate}
  \item The output halfspace $w_m$ is such that $\theta(w_m, u) \leq \frac{\theta}{2}$.
  \item The number of label queries is $O\del{\frac{d}{(1-2\eta)^2} \del{\ln\frac{d}{(1-2\eta)^2} + \ln\frac1\delta} }$.
  \item The number of unlabeled examples drawn is $O\del{\frac{d}{(1-2\eta)^3} \cdot \del{\ln\frac{d}{(1-2\eta)^2} + \ln\frac1\delta}^2 \cdot \frac1\theta}$.
  \item The algorithm runs in time $O\del{\frac{d^2}{(1-2\eta)^3} \cdot \del{\ln\frac{d}{(1-2\eta)^2} + \ln\frac1\delta}^2 \cdot \frac1\theta}$.
\end{enumerate}
\label{lem:mp-bn-restated}
\end{lemma}

\begin{proof}[Proof of Lemma~\ref{lem:mp-bn-restated}]
We show that each item holds with high probability respectively.
\begin{enumerate}
\item It can be verified that conditions for Lemma~\ref{lem:angle-shrinkage} are satisfied with $\zeta=1-2\eta$ (item 3 in the condition follows from Lemma~\ref{lem:delta-cos-geq-bn}, and item 4 in the condition follows from Lemma~\ref{lem:diff-theta-ub}). This shows that items 1 with probability at least $1-\delta/2$.

\item By the definition of $m$, the number of label queries is $m= O\del{\frac{d}{(1-2\eta)^{2}}\log\frac{d}{\delta(1-2\eta)^{2}}}$.

\item As for the number of unlabeled examples drawn by the algorithm, at each iteration $t \in [0,m]$, it takes
$Z_t$ trials to hit an example in $[\frac b 2, b]$, where $Z_t$ is a $\geom(p)$ random variable with
$p = \P_{x \sim D_\calX}[ w_t \cdot x \in [\frac b 2, b] ]$.
From Lemma~\ref{lem:1d-lb},
$p \geq \frac{\sqrt{d}}{8\pi} b = \frac{\tilde{c} (1-2\eta) \theta}{8 \pi}
= \Omega(\frac{(1-2\eta)\theta}{\ln\frac{d}{\delta(1-2\eta)^2}})$.

Define event
\[ E := \cbr{Z_1 + \ldots + Z_m \leq \frac{2m}{p}} \]
From Lemma~\ref{lem:geom-conc} and the choice of $m$, $\P[E] \geq 1 - \frac{\delta}{2}$.
Thus, on event $E$, the total number of unlabeled examples drawn is at most
$\frac {2m} p = O(\frac{d}{(1-2\eta)^{3}}\log^2\frac{d}{\delta(1-2\eta)^{2}}\frac{1}{\theta})$.

\item Observe that the time complexity for processing each example is at most $O(d)$.
This shows that on event $E$, the total running time of the algorithm is at most $O(d \cdot \frac {2m} p) = O(\frac{d^2}{(1-2\eta)^{3}}\log^2\frac{d}{\delta(1-2\eta)^{2}}\frac{1}{\theta})$.
\end{enumerate}

Therefore, by a union bound, with probability at least $1-\delta$, items 1 to 4 hold simultaneously.
\end{proof}

\begin{lemma}[Lemma~\ref{lem:mp-an} restated]
Suppose Algorithm~\ref{alg:modperceptron} has inputs labeling oracle $\calO$
that satisfies $\nu$-adversarial noise condition with respect to underlying halfspace $u$,
initial vector $w_0$ and
angle upper bound $\theta$ such that $\theta(w_0, u) \leq \theta$, confidence $\delta$,
number of iterations $m = \lceil (3200\pi)^3 d\ln\frac{(3200\pi)^3 d }{\delta} \rceil$,
band width
$b = \frac{1}{2(600\pi)^2\ln\frac{m^2}{\delta}} \cdot\frac{\theta}{\sqrt{d}}$.
Additionally $\nu \leq \frac{\theta}{384(600\pi)^4 \ln\frac{m^2}{\delta}}$.
Then with probability at least $1-\delta$:
\begin{enumerate}[leftmargin=1cm]
  \item The output halfspace $w_m$ is such that $\theta(w_m, u) \leq \frac{\theta}{2}$.
  \item The number of label queries is $O\del{d \cdot \del{\ln d + \ln\frac{1}{\delta} }}$.
  \item The number of unlabeled examples drawn is $O\del{d \cdot \del{\ln d + \ln\frac1\delta }^2 \cdot \frac1\theta }$
  \item The algorithm runs in time $O\del{d^2 \cdot \del{\ln d + \ln\frac1\delta }^2 \cdot \frac1\theta }$.
\end{enumerate}
\label{lem:mp-an-restated}
\end{lemma}

\begin{proof}[Proof of Lemma~\ref{lem:mp-an-restated}]
We show that each item holds with high probability respectively.
\begin{enumerate}
\item It can be verified that conditions for Lemma~\ref{lem:angle-shrinkage} are satisfied with $\zeta=1$ (item 3 in the condition follows from Lemma~\ref{lem:delta-cos-geq-an}, and item 4 in the condition follows from Lemma~\ref{lem:diff-theta-ub}). This gives items 1 with probability at least $1-\delta/2$.

\item By the definition of $m$, the number of label queries is $m= O\del{d \cdot \del{\ln d + \ln\frac{1}{\delta} }}$.

\item The number of unlabeled examples drawn by the algorithm can be analyzed similarly as in the previous proof, which is at most
$\frac{2m}{p} = O\del{d \cdot \del{\ln d + \ln\frac1\delta }^2 \cdot \frac1\theta }$ with probability at least $1-\delta/2$.

\item Observe that the time complexity for processing each example is at most $O(d)$.
This gives that on event $E$, the total running time of the algorithm is at most $O(d \cdot \frac {2m} p) = O\del{d^2 \cdot \del{\ln d + \ln\frac1\delta }^2 \cdot \frac1\theta }$.
\end{enumerate}

Therefore, by a union bound, with probability at least $1-\delta$, items 1 to 4 hold simultaneously.
\end{proof}

Next we show a technical lemma used in the above proofs, coarsely bounding the difference between $\cos\theta_{t+1}$
and $\cos\theta_t$.
\begin{lemma}
\label{lem:diff-theta-ub}
Suppose $0<\tilde{c}, \zeta<1$, $b=\frac{\tilde{c}\zeta\theta}{\sqrt{d}}\leq1$, and $(x_t, y_t)$
is drawn from distribution $D|_{R_t}$ where $R_{t}=\left\{(x,y):x\cdot w_t\in[\frac{b}2, b]\right\}$.
If unit vector $w_t$ has angle $\theta_{t}$ with $u$ such that $\theta_t \leq \frac53\theta$,
then update~\eqref{eqn:modperceptron}
has the following guarantee:
 $\left|\cos\theta_{t+1}-\cos\theta_{t}\right|\leq\frac{16\tilde{c}\zeta\theta^{2}}{3\sqrt{d}}$.
\end{lemma}
\begin{proof}
By Lemma~\ref{lem:coschange},
\[ \cos\theta_{t+1} - \cos\theta_t = - 2\one\cbr{y_t \neq \sign(w_t \cdot x_t)} (w_t \cdot x_t) \cdot (u \cdot x_t). \]

Firstly, note $\left|\cos\theta_{t+1}-\cos\theta_{t}\right|\leq2\left|w_t\cdot x_{t}\right|\left|u \cdot x_{t}\right|\leq2b\left|u \cdot x_{t}\right|$.

Observe that
\begin{eqnarray*}
&&\left|u \cdot x_{t}\right| \\
&\leq& \left|w_t \cdot x_{t}\right| + \left|(u - w_t) \cdot x_{t}\right| \\
&\leq& b + 2\sin\frac{\theta_t}{2} \\
&\leq& b + \theta_t
\end{eqnarray*}

Thus, we have
$\left|\cos\theta_{t+1}-\cos\theta_{t}\right|\leq 2b(b+\theta_{t})
=\frac{2\tilde{c}^{2}\zeta^{2}\theta^{2}}{d}+\frac{2\tilde{c}\zeta\theta\theta_{t}}{\sqrt{d}}
\leq\frac{16\tilde{c}\zeta\theta^{2}}{3\sqrt{d}}$.
\end{proof}

\begin{lemma}
Suppose $0<\zeta<1$, and the following conditions hold:
\begin{enumerate}
\item Initial unit vector $w_0$ has angle
$\theta_0 = \theta(w_0, u) \leq \theta \leq \frac{27}{50}\pi$ with $u$;
\item Integer $m = \lceil \frac{(3200\pi)^3 d }{\zeta^{2}} (\ln\frac{(3200\pi)^3 d }{\zeta^{2}} + \ln\frac{1}{\delta}) \rceil$
      and $\tilde{c} = \frac{1}{2 (600\pi)^2 \ln\frac{m^2}{\delta}}$;
\item For all $t$, if $\frac14 \theta \leq \theta_t \leq \frac53 \theta$, then
      $\E[\cos\theta_{t+1} - \cos\theta_t|\theta_t] \geq \frac{\tilde{c}}{100\pi} \frac{\zeta^2\theta^2}{d}$;
\item For all $t$, if $\theta_t \leq \frac 5 3\theta$, then $|\cos\theta_{t+1} - \cos\theta_t| \leq \frac{16\tilde{c}\zeta\theta^2}{3\sqrt{d}}$ holds with probability 1.
%$\cos\theta_{t+1} - \cos\theta_t$ has subgaussian tails, i.e.
%			$\P[\cos\theta_{t+1} - \cos\theta_t \geq \frac{(1-2\eta)^2\theta^2 s}{d}] \leq \exp(-\frac{s^2}{6})$
\end{enumerate}
Then with probability at least $1-\delta/2$, $\theta_m \leq \frac12\theta$.
\label{lem:angle-shrinkage}
\end{lemma}

\begin{proof}
 Define random variable $D_t$ as:
\[ D_t := \del{\cos\theta_{t+1} - \cos\theta_t - \frac{\tilde{c}}{100\pi} \frac{\zeta^2 \theta^2}{d}} \one\cbr{\frac14 \theta \leq \theta_t \leq \frac53 \theta}\]
Note that $\E[D_t | \theta_t] \geq 0$ and from Lemma~\ref{lem:diff-theta-ub},
$|D_t| \leq |\cos\theta_{t+1} - \cos\theta_t| + \frac{\tilde{c}}{100\pi} \frac{\zeta^2 \theta^2}{d} \leq \frac{6\tilde{c}\zeta\theta^2}{\sqrt{d}}$.
Therefore, $\cbr{D_t}$ is a bounded submartingale difference sequence.
By Azuma's Inequality (see Lemma~\ref{lem:azuma}) and union bound,
define event
\[ E = \cbr{\text{for all } 0 \leq t_1 \leq t_2 \leq m, \sum_{s=t_1}^{t_2-1} D_s \geq -\frac{6\tilde{c}\zeta\theta^2}{\sqrt{d}} \sqrt{2(t_2 - t_1)\ln\frac{2m^2}{\delta}} }\]
Then $\P(E) \geq 1-\frac{\delta}2$.

We now condition on event $E$.
We break the subsequent analysis into two parts:
(1) Show that there exists some $t$ such that $\theta_t$ goes below $\frac14\theta$.
(2) Show that $\theta_t$ must stay below $\frac12\theta$ afterwards.

\begin{enumerate}
\item First, it can be checked by algebra that
$m \geq \frac{200\pi d}{\zeta^2 \tilde{c}}$.
%We first show a technical claim given a lower bound on $m$.
%\begin{claim}
%$m \geq \frac{200\pi d}{(1-2\eta)^2\theta^2 \tilde{c}}$.
%\label{claim:m-lb}
%\end{claim}
%\begin{proof}
%Expanding the definition of $\tilde{c}$,
%it suffices to show $m \geq \frac{200\pi d \cdot 2(600\pi)^2 \ln\frac{m^2}{\delta}}{(1-2\eta)^2\theta^2}$.
%To this end, define $M = \max\cbr{t: t <  \frac{200\pi d \cdot 2(600\pi)^2 \ln\frac{t^2}{\delta}}{(1-2\eta)^2\theta^2}}$.
%Our goal now becomes showing $M < m$.
%For every $t < \frac{200\pi d \cdot 2(600\pi)^2 \ln\frac{t^2}{\delta}}{(1-2\eta)^2\theta^2}$, we have
%$t < \frac{400\pi d \cdot 2(600\pi)^2 \ln\frac1{\delta}}{(1-2\eta)^2\theta^2}$ or $t < \frac{800\pi d \cdot 2(600\pi)^2 \ln t}{(1-2\eta)^2\theta^2}$.
%The rest of the inequality follows from algebra and using $t \ln A \leq t \ln (2t) + \frac A 2$.
%\end{proof}
%Now we are ready show the following claim.
We show the following claim.
\begin{claim}
There exists some $t \in [0,m]$, such that $\theta_t < \frac14\theta$.
\label{claim:tsmallangle}
\end{claim}
\begin{proof}
We first show that it is impossible for all $t \in [0,m]$
such that $\theta_t \in \intcc{\frac14\theta, \frac53\theta}$.
To this end, assume this holds for the sake of contradiction.
In this case, for all $t \in [0,m]$,
$D_t = \cos\theta_{t+1} - \cos\theta_t - \frac{\tilde{c}}{100\pi} \frac{\zeta^2 \theta^2}{d}$.
Therefore,
\begin{eqnarray*}
	 &&\cos\theta_m - \cos\theta_0 \\
   &=& \sum_{s=0}^{m-1} D_s + \frac{\tilde{c}}{100\pi} \frac{\zeta^2 \theta^2}{d} m \\
   &\geq& \frac{\tilde{c}}{100\pi} \frac{\zeta^2 \theta^2}{d} m -
		     \frac{6\tilde{c}\zeta\theta^2}{\sqrt{d}} \sqrt{2 m \ln\frac{m^2}{\delta}} \\
	 &\geq& \frac{\theta^2}{100\pi} \sbr{\frac{\tilde{c} \zeta^2 m }{d} - \sqrt{\frac{\tilde{c} \zeta^2 m }{d}} } \\
	 &\geq& \theta^2
\end{eqnarray*}
where the first inequality is from the definition of event $E$, the second
inequality is from that $\tilde{c} = \frac{1}{2(600\pi)^2 \ln\frac{m^2}{\delta}}$,
the third inequality is from that
$\frac{\tilde{c}\zeta^2m}{d} \geq 200\pi$.
%by Claim~\ref{claim:m-lb}

Since $\cos\theta_0 \geq \cos\theta \geq 1 - \frac12\theta^2$, this gives
that $\cos\theta_m \geq 1 + \frac12\theta^2 > 1$, contradiction.

Next, define $\tau := \min\cbr{t \geq 0: \theta_t \notin \intcc{\frac14\theta, \frac53\theta}}$.
We now know that
$\tau \leq m$ by the reasoning above. It suffices to show that $\theta_\tau < \frac14\theta$, that is,
the first time when $\theta_t$ goes outside the interval $[\frac14\theta, \frac53\theta]$, it must
be crossing the left boundary as opposed to the right one.

By the definition of $\tau$, for all $0 \leq t \leq \tau - 1$, $\theta_\tau \in \intcc{\frac14\theta, \frac53\theta}$.
Thus,
\begin{eqnarray}
&&\cos\theta_\tau - \cos\theta_0 \nonumber\\
&=& \sum_{t=0}^{\tau-1} D_t + \frac{\tilde{c}}{100\pi} \frac{\zeta^2 \theta^2}{d} \tau \nonumber\\
&\geq& \frac{\tilde{c}}{100\pi} \frac{\zeta^2 \theta^2}{d} \tau -\frac{6\tilde{c}\zeta\theta^2}{\sqrt{d}} \sqrt{\tau \ln\frac{m^2}{\delta}} \nonumber\\
&\geq& -900\pi\ln\frac{m^2}{\delta} \tilde{c} \theta^2 \geq -\frac1{75} \theta^2
\label{eqn:theta-tau-theta-0}
\end{eqnarray}
where the first inequality is by the definition of $E$; the second inequality
is by minimizing over $\tau \in [0,m]$; the last inequality is from the definition
of $\tilde{c}$.

Now, if $\theta_\tau \geq \frac53\theta$, then
\begin{eqnarray*}
	\cos \theta_\tau - \cos \theta_0 & \leq & \cos \frac53\theta - \cos \theta \\
	&\leq& 1-\frac15 \left( \frac53\right)^2 \theta^2 - 1 + \frac12\theta^2 \\
	&<& -\frac{1}{75}\theta^2
\end{eqnarray*}
where the first inequality follows from $\theta_\tau \geq \frac53\theta$ and $\theta_0 \leq \theta$, and the second inequality follows from Lemma~\ref{lem:cos}. This contradicts with Inequality~\eqref{eqn:theta-tau-theta-0}.

This gives that $\theta_\tau < \frac53\theta$. Since $\theta_\tau \notin \intcc{\frac14\theta, \frac53\theta}$, it must be the case that
$\theta_\tau < \frac14\theta$.
\end{proof}

\item We now show the following claim to conclude the proof.
%We now focus on the last $t$ such that $\theta_t < c_1\theta$.
\begin{claim}
$\theta_m$, the angle in the last iteration, is at most $\frac12\theta$.
\end{claim}
\begin{proof}
Define $\sigma = \max\cbr{t \in [0,m]: \theta_t < \frac14\theta}$. by Claim~\ref{claim:tsmallangle},
such $\sigma$ is well-defined on event $E$. We now show that $\theta_t$ will not exceed $\frac12\theta$ afterwards.
Assume for the sake of contradiction that for some $t > \sigma$, $\theta_t > \frac12\theta$.

Now define $\gamma := \min\cbr{t > \sigma: \theta_t > \frac12\theta}$. We know by the definitions
of $\sigma$ and $\gamma$, for all
$t \in [\sigma+1, \gamma-1]$, $\theta_t \in [\frac14\theta, \frac12\theta]$. Thus,
\begin{eqnarray}
&&\cos\theta_\gamma - \cos\theta_{\sigma+1} \nonumber\\
&=& \sum_{t=\sigma+1}^{\gamma-1} D_t + \frac{\tilde{c}}{100\pi} \frac{\zeta^2 \theta^2}{d} (\gamma-\sigma-1) \nonumber\\
&\geq& \frac{\tilde{c}}{100\pi} \frac{\zeta^2 \theta^2}{d} (\gamma - \sigma-1) -\frac{6 \tilde{c}\zeta\theta^2}{\sqrt{d}} \sqrt{(\gamma - \sigma-1) \ln\frac{m^2}{\delta}} \nonumber\\
&\geq& -900\pi\ln\frac{m^2}{\delta} \tilde{c} \geq -\frac1{75} \theta^2
\label{eqn:theta-tau-theta-sigma}
\end{eqnarray}
where the first inequality is by the definition of $E$; the second inequality
is by minimization over $\gamma - \sigma - 1 \in [0,m]$; the last inequality is from the definition
of $\tilde{c}$.

On the other hand, $\theta_\gamma > \frac12\theta$ and $\theta_\sigma < \frac14\theta$. We have
\begin{eqnarray*}
\cos\theta_\gamma - \cos\theta_{\sigma+1} & \leq & \cos\theta_\gamma - \cos\theta_{\sigma} + \frac{6 \tilde{c}\zeta\theta^2}{\sqrt{d}} \\
&\leq& \cos\frac{\theta}{2} - \cos\frac{\theta}{4} + \frac{6 \tilde{c}\zeta\theta^2}{\sqrt{d}} \\
&\leq& 1-\frac{1}{20}\theta^2 - 1 + \frac{1}{32}\theta^2 + \frac{6 \tilde{c}\zeta\theta^2}{\sqrt{d}} \\
&<& -\frac{1}{75}\theta^2
\end{eqnarray*}
where the first inequality follows from Lemma~\ref{lem:diff-theta-ub}, the third follows from Lemma~\ref{lem:cos}, and the last follows from algebra. This contradicts with Inequality~\eqref{eqn:theta-tau-theta-sigma}.
\end{proof}
\end{enumerate}
Thus, with probability at least $1-\delta/2$, $\theta_m \leq \frac12\theta$.
\end{proof}

\section{Progress Measure Analysis}
\label{sec:progmeasure}
In this section, we prove two key lemmas on $\cos\theta_t$ (Lemmas~\ref{lem:delta-cos-geq-bn} and~\ref{lem:delta-cos-geq-an}),
our measure of progress.
We show that under the bounded noise model and the adversarial noise model, $\cos\theta_t$ increases by
a decent amount in expectation at each iteration of \MP, with appropriate settings of
bandwidth $b$.

We begin with a generic lemma that gives a recurrence of $\cos\theta_t$ when the modified Perceptron
update rule~\eqref{eqn:modperceptron-initial} is applied to a new example.
\begin{lemma}
Suppose $w_t \in \R^d$ is a unit vector, and $(x_t, y_t)$ is an labeled example where $x_t \in \R^d$ is
a unit vector and $y_t \in \cbr{-1,+1}$. Let $\theta_t = \theta(u, w_t)$. Then, update
\begin{equation}
	w_{t+1} \gets w_t - 2\one\cbr{y_t w_t \cdot x_t < 0} (w_t \cdot x_t) \cdot x_t
	\label{eqn:modperceptron}
\end{equation}
gives an unit vector $w_{t+1}$ such that
\begin{equation}
 \cos\theta_{t+1} = \cos\theta_t - 2\one\cbr{y_t w_t \cdot x_t < 0} (w_t \cdot x_t) \cdot (u \cdot x_t)
 \label{eqn:coschange}
 \end{equation}
\label{lem:coschange}
\end{lemma}
\begin{proof}
We first show that $w_{t+1}$ is still a unit vector.
If $y_t = \sign(w_t \cdot x_t)$, then $w_{t+1} = w_t$, thus it is still a unit vector; otherwise
$w_{t+1} = w_t - 2 (w_t \cdot x_t) \cdot x_t$. This gives that
\[ \|w_{t+1}\|^2 = \|w_t\|^2 - 4 (w_t \cdot x_t) (w_t \cdot x_t) + \| 2(w_t \cdot x_t) \cdot x_t \|^2 = \|w_t\|^2 = 1.\]

This implies that $\cos\theta_t = w_t \cdot u$, and $\cos\theta_{t+1} = w_{t+1} \cdot u$. Now, taking inner products with $u$
on both sides of Equation~\eqref{eqn:modperceptron}, we get
\[ w_{t+1} \cdot u = w_t \cdot u - 2\one\cbr{y_t w_t \cdot x_t < 0} (w_t \cdot x_t) \cdot (u \cdot x_t)\]
which is equivalent to Equation~\eqref{eqn:coschange}.
\end{proof}

\subsection{Progress Measure under Bounded Noise}
\label{sec:pf-progress-bn}
\begin{lemma}[Progress Measure under Bounded Noise]
\label{lem:delta-cos-geq-bn}
Suppose $0<\tilde{c}<\frac1{288}$, $b = \frac{\tilde{c}(1-2\eta)\theta}{\sqrt{d}}$,
$\theta\leq\frac{27}{50}\pi$, and $(x_t, y_t)$ is drawn from $D|_{R_t}$, where
$R_{t}=\left\{ (x, y): x\cdot w_t\in[\frac{b}{2},b]\right\}$. Meanwhile, the oracle $\calO$ satisfies
the $\eta$-bounded noise condition.
If unit vector $w_t$ has angle $\theta_t$ with $u$ such that
$\frac14 \theta \leq \theta_{t} \leq \frac53\theta$,
then update~\eqref{eqn:modperceptron}
has the following guarantee:
\[ \mathbb{E}\left[\cos\theta_{t+1}-\cos\theta_{t}\mid\theta_{t}\right]\geq\frac{\tilde{c}}{100\pi}\frac{(1-2\eta)^{2}\theta^{2}}{d}. \]
\end{lemma}
\begin{proof}
Define random variable $\xi=x_{t}\cdot w_t$. By the tower property of conditional expectation,
$\mathbb{E}\left[\cos\theta_{t+1}-\cos\theta_{t}\mid\theta_{t}\right]
=\mathbb{E}\left[\mathbb{E}\left[\cos\theta_{t+1}-\cos\theta_{t}\mid\theta_{t},\xi \right]\mid\theta_{t}\right]$.
Thus, it suffices to show
$$\mathbb{E}\left[\cos\theta_{t+1}-\cos\theta_{t}\mid\theta_{t},\xi\right]\geq\frac{\tilde{c}}{100\pi}\frac{(1-2\eta)^{2}\theta^{2}}{d}$$
for all $\theta_{t} \in [\frac14\theta, \frac53\theta]$ and $\xi \in [\frac12 b, b]$.

By Lemma~\ref{lem:coschange}, we know that
\[ \cos\theta_{t+1} - \cos\theta_t = - 2\one\cbr{y_t \neq \sign(w_t \cdot x_t)} (w_t \cdot x_t) \cdot (u \cdot x_t). \]

We simplify $\E\left[\cos\theta_{t+1}-\cos\theta_{t}\mid\theta_{t},\xi\right]$ as follows:
\begin{eqnarray}
 && \E\left[\cos\theta_{t+1}-\cos\theta_{t}\mid\theta_{t},\xi\right] \nonumber \\
&= & \E\left[-2\xi u\cdot x_{t} \mathds{1}\cbr{y_t = -1} \mid\theta_{t},\xi\right] \nonumber \\
&= & \E\left[-2\xi u\cdot x_{t} (\mathds{1}\{u\cdot x_{t}>0,y_t = -1\} + \mathds{1}\{u\cdot x_{t}<0, y_t = -1\})\mid\theta_{t},\xi\right]\nonumber \\
&\geq & \E\left[-2\xi u\cdot x_{t} (\eta \mathds{1}\{u\cdot x_{t}>0\} + (1-\eta)\mathds{1}\{u\cdot x_{t}<0\})\mid\theta_{t},\xi\right]\nonumber \\
&= & \E\left[-2\xi u\cdot x_{t} (\eta + (1-2\eta)\mathds{1}\{u\cdot x_{t}<0\})\mid\theta_{t},\xi\right]\nonumber \\
&= & -2\xi \left(\eta \E\left[u\cdot x_{t}  \mid\theta_{t},\xi\right] + (1-2\eta) \E\left[u\cdot x_{t}\mathds{1}\{u\cdot x_{t}<0\}  \mid\theta_{t},\xi\right]\right)
\label{eqn:bn-progress-decomp}
\end{eqnarray}
where the second equality is from algebra, the first inequality is from that $\P[y_t = -1| u \cdot x_t > 0] \leq \eta$ and $\P[y_t = -1| u \cdot x_t < 0] \geq 1 - \eta$,
the last two equalities are from algebra.

By Lemma~\ref{lem:cond-moment} and that $0 \leq \theta_t \leq \frac53\theta \leq \frac9{10}\pi$, $\E[ u\cdot x_{t} |\theta_{t},\xi ] \leq \xi$ and $\E[ u\cdot x_{t}\one\cbr{u\cdot x_{t}<0}|\theta_{t},\xi ] \leq \xi - \frac{\theta_t}{36\sqrt{d}}$.
%Conditioned on $\xi$, the data $x_t$ is distributed uniformly on the sphere section
%$\cbr{x: \|x\| = 1, w_t \cdot x = \xi}$.

Thus,
\begin{eqnarray*}
	&& \mathbb{E}\left[\cos\theta_{t+1}-\cos\theta_{t}\mid\theta_{t},\xi \right] \\
 &\geq& -2\xi (\xi \eta + (\xi - \frac{\theta_t}{36\sqrt{d}})(1-2\eta)) \\
	&\geq& 2\xi (\frac{\theta_t}{36\sqrt{d}}(1-2\eta) - \xi) \\
	&\geq& b \frac{\theta_t}{72\sqrt{d}}(1-2\eta) \\
	&\geq& \frac{\tilde{c}}{100\pi}\frac{(1-2\eta)^{2}\theta^{2}}{d}
\end{eqnarray*}
where the first and second inequalities are from algebra, the third inequality is from that
$\xi \leq b \leq \frac{\theta(1-2\eta)}{288\sqrt{d}} \leq \frac{\theta_t(1-2\eta)}{72\sqrt{d}}$,
and that $\xi \geq \frac{b}{2}$.
the last inequality is by expanding $b = \frac{\tilde{c}(1-2\eta)\theta}{\sqrt{d}}$
and that $\theta_t \geq \frac{\theta}{4}$.

In conclusion, if $\frac14 \theta \leq \theta_t \leq \frac53 \theta$, then
$\mathbb{E}\left[\cos\theta_{t+1}-\cos\theta_{t}\mid\theta_{t},\xi\right] \geq \frac{\tilde{c}}{100\pi}\frac{(1-2\eta)^{2}\theta^{2}}{d}$ for $\xi \in [\frac{b}2, b]$.
The lemma follows.
\end{proof}

\subsection{Progress Measure under Adversarial Noise}
\label{sec:pf-progress-an}

\begin{lemma}[Progress Measure under Adversarial Noise]
\label{lem:delta-cos-geq-an}
Suppose $0 \leq \tilde{c} \leq \frac1{100\pi}$,
$b = \frac{\tilde{c}\theta}{\sqrt{d}}$, $\theta \leq \frac{27}{50}\pi$,
and $(x_t, y_t)$ is drawn from distribution $D|_{R_t}$ where
$R_t = \cbr{(x,y): x \cdot w_t \in [\frac{b}{2}, b]}$.
Meanwhile, the oracle $\calO$ satisfies the $\nu$-adversarial noise condition
where $\nu \leq \frac{\tilde{c}\theta}{192(200\pi)^2}$.
If unit vector $w_t$ has angle $\theta_t$ with $u$ such that
$\frac14 \theta \leq \theta_{t} \leq \frac53\theta$,
then update~\eqref{eqn:modperceptron}
has the following guarantee:
\[ \mathbb{E}\left[\cos\theta_{t+1}-\cos\theta_{t}\mid\theta_{t}\right]\geq\frac{\tilde{c}}{100\pi}\frac{\theta^{2}}{d}. \]
\end{lemma}

\begin{proof}
Define random variable $\xi = x_t \cdot w_t$.

By Lemma~\ref{lem:coschange}, we know that
\[ \cos\theta_{t+1} - \cos\theta_t = - 2\one\cbr{y_t \neq \sign(w_t \cdot x_t)} (w_t \cdot x_t) \cdot (u \cdot x_t). \]

We expand $\E\left[\cos\theta_{t+1}-\cos\theta_{t}\mid\theta_{t} \right]$ as follows.
\begin{eqnarray}
	&& \E\left[\cos\theta_{t+1}-\cos\theta_{t}\mid\theta_{t} \right] \nonumber \\
	&= & \E\left[-2 (w_t \cdot x_t) (u\cdot x_{t}) \mathds{1}\cbr{y_t=-1} \mid\theta_{t} \right] \nonumber \\
	&= & \E\left[-2 (w_t \cdot x_t) (u\cdot x_{t}) \mathds{1}\cbr{u \cdot x_t<0} \mid\theta_{t} \right] \nonumber \\
	&  & + \E\left[2 (w_t \cdot x_t) (u\cdot x_{t}) (\mathds{1}\cbr{y_t =+1, u \cdot x_t<0} - \mathds{1}\cbr{y_t =-1, u \cdot x_t>0})) \mid\theta_{t} \right]
\label{eqn:an-progress-decomp}
\end{eqnarray}
We bound the two terms separately. Firstly,
	\begin{eqnarray}
		&& \E\left[ -2 (w_t \cdot x_t) (u \cdot x_{t}) \mathds{1}\cbr{u \cdot x_t<0} \mid\theta_{t} \right] \nonumber \\
		&\geq& - b \E\left[  (u \cdot x_{t}) \mathds{1}\cbr{u \cdot x_t<0} \mid\theta_{t} \right] \nonumber \\
		&=& - b \E\sbr{ \E\sbr{ (u \cdot x_{t}) \mathds{1}\cbr{u \cdot x_t<0} \mid\theta_{t}, b }\mid\theta_{t} } \nonumber \\
		&\geq& b (\frac{\theta_t}{36 \sqrt{d}} - b)
		\label{eqn:an-progress-term1}
	\end{eqnarray}
where the first inequality is from that $-(u \cdot x_{t}) \mathds{1}\cbr{u \cdot x_t<0} \geq 0$ and $w_t \cdot x_t \geq \frac{b}{2}$,
the equality is from the tower property of conditional expectation,
the second inequality is from Lemma~\ref{lem:cond-moment}.

Secondly,
\begin{eqnarray}
	&& \left|\E\left[ 2(w_t \cdot x_{t}) (u\cdot x_{t}) (\mathds{1}\cbr{y_t = +1, u \cdot x_t<0} - \mathds{1}\cbr{y_t = -1, u \cdot x_t>0})) \mid\theta_{t} \right]\right| \nonumber \\
	&\leq& 2b \E\left[ |u\cdot x_{t}| \mathds{1}\cbr{y_t \neq \sign(u \cdot x_t))} \mid\theta_{t} \right] \nonumber \\
	&\leq& 2b \sqrt{\E\left[ \mathds{1}\cbr{y_t \neq \sign(u \cdot x_t))} \mid\theta_{t} \right] \cdot \E\left[ (u\cdot x_{t})^2 \mid \theta_{t} \right]} \nonumber \\
	&=& 2b \sqrt{\P\left[ y_t \neq \sign(u \cdot x_t) \mid\theta_{t} \right] \E\left[ \E\left[ (u\cdot x_{t})^2 \mid \theta_{t},\xi\right] | \theta_t\right]}
	\label{eqn:cs}
\end{eqnarray}
where the first inequality is from that $|\E[X]| \leq \E|X|$, and $w_t \cdot x_t \leq b$,
the second inequality is from Cauchy-Schwarz, the third equality is by algebra.

Now we look at the two terms inside the square root. First,
\begin{eqnarray*}
	&&\P\left[ y_t \neq \sign(u \cdot x_t) \mid\theta_{t} \right] \\
	&=&\P_{x \sim D|_{R_t}}\left[ y \neq \sign(u \cdot x) \right] \\
	&\leq& \frac{\P_{(x,y) \sim D}\left[ y \neq \sign(u \cdot x) \right]}{\P_{x \sim D} \left[ x_1 \in [b/2, b]  \right] } \\
	&\leq& \frac{8\pi \nu }{\tilde{c} \theta} \\
	&\leq& \frac{1}{16(200\pi)^2}
\end{eqnarray*}
where the first inequality is from that $\P[A|B] \leq \frac{\P[A]}{\P[B]}$,
the second inequality is from Lemma~\ref{lem:1d-lb} that
$\P_{x \sim D} \left[ x_1 \in [b/2, b]  \right] \geq \frac{\sqrt{d}}{8\pi}b = \frac{\tilde{c}\theta}{8\pi}$
, and the last inequality is by our assumption on $\nu$.

Second, fix $\xi \in [\frac{b}{2}, b]$, $\xi \leq b \leq \frac{\theta_t}{4\sqrt{d}}$. Item 2 of Lemma~\ref{lem:cond-moment}
implies that $\E\left[ (u\cdot x_{t})^2 \mid \theta_{t},\xi\right] \leq \frac{5\theta_t^2}{d}$.
By the tower property of conditional expectation,
$\E\left[ (u\cdot x_{t})^2 \mid \theta_t \right] \leq \frac{5\theta_t^2}{d}$.
Continuing Equation~\eqref{eqn:cs}, we get
\begin{equation}
	 \left|\E\left[ 2(w_t \cdot x_t) (u\cdot x_{t}) (\mathds{1}\cbr{y_t = +1, u \cdot x_t<0} - \mathds{1}\cbr{y_t = -1, u \cdot x_t>0})) \mid\theta_{t} \right]\right|
   \leq b \frac{\theta_t}{100\pi\sqrt{d}}.
	 \label{eqn:an-progress-term2}
\end{equation}

Continuing Equation~\eqref{eqn:an-progress-decomp}, we have
\begin{eqnarray*}
	&&\E\left[\cos\theta_{t+1}-\cos\theta_{t}\mid\theta_{t} \right] \\
	&\geq& b (\frac{\theta_t}{36\sqrt{d}} - \frac{\theta_t}{100\pi\sqrt{d}} - b) \\
	&\geq& b \frac{\theta_t}{25\pi\sqrt{d}} \geq \frac{\tilde{c}}{100\pi} \frac{\theta^{2}}{d}
\end{eqnarray*}
where the first inequality is from Equations~\eqref{eqn:an-progress-term1}
and~\eqref{eqn:an-progress-term2}, the second inequality is from algebra and
that $b \leq \frac{\theta_t}{100\pi\sqrt{d}}$, the third inequality is by expanding
$b = \frac{\tilde{c}\theta}{\sqrt{d}}$ and $\theta_t \geq \frac{\theta}{4}$.
\end{proof}

\section{Acute Initialization}
\label{sec:pf-init}
\label{sec:acute}

We show in this section that the angle between the
initial vector $v_0$ and the underlying halfspace $u$ can be assumed to be acute
 under the two noise settings without loss of generality. To this end, we give two algorithms (Algorithms~\ref{alg:init-bn} and~\ref{alg:init-an})
 that returns a halfspace that has angle at most $\frac\pi 4$ with $u$, with constant
 overhead in label and time complexities. The techniques here are
 due to Appendix B of~\cite{ABL14}.
 This fact, in conjunction with Theorems~\ref{thm:ap-bn} and~\ref{thm:ap-an}, yield an active learning
 algorithm that learns the target halfspace unconditionally with a constant overhead of label and time complexities.

For the bounded noise setting,
we construct Algorithm~\ref{alg:init-bn} as an initialization procedure.
It runs $\AP$ twice, taking a vector $v_0$ and its opposite direction $-v_0$ as initializers.
Then it performs hypothesis
testing using $\tilde{O}(\frac{1}{(1-2\eta)^2})$ labeled examples
to identify a halfspace that has angle at most $\frac\pi 4$ with $u$.

\begin{algorithm}
\caption{Master Algorithm in the Bounded Noise Setting}
\begin{algorithmic}[1]
\REQUIRE Labeling oracle $\calO$, confidence $\delta$, noise upper bound $\eta$, sample schedule $\cbr{m_k}$, band width $\cbr{b_k}$.
\ENSURE a halfspace $\hat{v}$ such that $\theta(\hat{v}, u) \leq \frac{\pi}{4}$.
		\STATE $v_0 \leftarrow (1,0,\ldots,0)$.
		\STATE $v_+ \leftarrow \AP(\calO, v_0, \frac{(1-2\eta)}{16},\frac{\delta}{3}, \cbr{m_k},
		 \cbr{b_k})$.
		\STATE $v_- \leftarrow \AP(\calO, -v_0, \frac{(1-2\eta)}{16},\frac{\delta}{3}, \cbr{m_k},
		\cbr{b_k})$.
		%\IF{$\theta(v_+, v_-) \leq \epsilon/2$}
		%	\RETURN $v_+$
		\STATE Define region $R := \cbr{x: \sign(v_+ \cdot x) \neq \sign(v_- \cdot x) }$.
		\STATE $S \gets$ Draw $\frac{8}{(1-2\eta)^2}\ln\frac{6}{\delta}$ iid examples from $D|_R$ and query their labels.
		\IF{$\err_S(h_{v_+}) \leq \err_S(h_{v_-})$}
			\RETURN $v_+$
		\ELSE
			\RETURN $v_-$
		\ENDIF
\end{algorithmic}
\label{alg:init-bn}
\end{algorithm}

%We show Algorithm~\ref{alg:init-bn}

\begin{theorem}
Suppose Algorithm~\ref{alg:init-bn} has inputs labeling oracle $\calO$ that satisfies
$\eta$-bounded noise condition with respect to $u$, confidence $\delta$,
sample schedule $\cbr{m_k}$
where $m_k = \Theta\del{\frac{ d }{(1-2\eta)^{2}} (\ln\frac{ d }{(1-2\eta)^{2}} + \ln\frac{k}{\delta}) }$,
band width $\cbr{b_k}$ where $b_k = \tilde{\Theta}\del{\frac{ 2^{-k} (1-2\eta)}{\sqrt{d}}}$.
Then, with probability at least  $1-\delta$, the output $\hat{v}$ is such that
$\theta(\hat{v}, u) \leq \frac\pi 4$. Furthermore,
(1) the total number of label queries to oracle $\calO$ is at most
$\tilde{O}\del{\frac{d}{(1-2\eta)^2} }$; (2) the total number of unlabeled examples drawn is $\tilde{O}\del{\frac{d}{(1-2\eta)^3}}$; (3) the algorithm runs in time $\tilde{O}\del{\frac{d^2}{(1-2\eta)^3}}$.
\label{thm:init-bn}
\end{theorem}

\begin{proof}
Note that one of $\theta(v_0, u)$, $\theta(-v_0, u)$ is at most $\frac{\pi}{2}$.
From Theorem~\ref{thm:ap-bn} and union bound, we know that with probability at least  $1-\frac{2\delta}{3}$,
either
$\theta(v_+, u) \leq \frac{(1-2\eta) \pi} {16}$, or $\theta(v_-, u) \leq \frac{(1-2\eta) \pi} {16}$.

Suppose without loss of generality, $\theta(v_+, u) \leq \frac{(1-2\eta)\pi} {16}$.
We consider two cases.

\paragraph{Case 1: $\theta(v_+, v_-) \leq \pi / 8$.}
By triangle inequality, $\theta(v_-, u) \leq \theta(v_+, u) + \theta(v_+, v_-) \leq \pi / 4$.
In this case,
$\theta(v_+, u) \leq \frac \pi 4$ and $\theta(v_-, u) \leq \frac \pi 4$ holds simultaneously.
Therefore, the returned vector $\hat{v}$ satisfies $\theta(\hat{v}, u) \leq \frac\pi 4$.

\paragraph{Case 2: $\theta(v_+, v_-) > \pi / 8$.} In this case,
$\P[x \in R] \geq 1/8$, thus,
%$\theta(v_-, u) \geq \epsilon/2 - (1-2\eta) \epsilon / 4 \geq \epsilon/4$.
\[ \P_R[\sign(v_+\cdot x) \neq \sign(u \cdot x)] \leq \frac{\P[\sign(v_+\cdot x) \neq \sign(u \cdot x)]}{\P[x \in R]} \leq \frac{1-2\eta} 8 = \frac 1 4 (\frac 1 2 - \eta).\]
Meanwhile, $\P_R[\sign(v_+\cdot x) \neq y] \leq \eta \P_R[\sign(v_+\cdot x) = \sign(u \cdot x)] + \P_R[\sign(v_+\cdot x) \neq \sign(u \cdot x)]$.
Therefore,
\begin{eqnarray*}
&& \frac12 - \P_R[\sign(v_+\cdot x) \neq y]\\
&\geq& (\frac12 - \eta) \P_R[\sign(v_+\cdot x) = \sign(u \cdot x)] - \frac12\P_R[\sign(v_+\cdot x) \neq \sign(u \cdot x)] \\
&\geq& (\frac12 - \eta) \cdot \frac12 - (\frac12 - \eta) \cdot \frac14 \\
&\geq& \frac14 (\frac12 - \eta)  \\
\end{eqnarray*}
Since $v_+$ disagrees with $v_-$ everywhere on $R$, $\P_R[\sign(v_+\cdot x) \neq y] + \P_R[\sign(v_-\cdot x) \neq y] = 1$.
Thus, $\err_{D|_R}(h_{v_+}) \leq \frac12 - (\frac12 - \eta) \frac14$ and $\err_{D|_R}(h_{v_-}) \geq \frac12 + (\frac12 - \eta) \frac14$.
Therefore, by Hoeffding's Inequality, with probability at least $1-\delta/3$,
\[ \err_S(v_+) < \frac12 < \err_S(v_-) \]
therefore $v_+$ will be selected for $\hat{v}$. This shows that $\theta(\hat{v}, u) \leq \pi/4$.

In conclusion, by union bound, we have shown that with probability $1-\delta$, $\theta(\hat{v}, u) \leq \frac \pi 4$.
%This immediately implies that $\P[\sign(\hat{v} \cdot x) \neq \sign(u \cdot x)] \leq \frac14$.
The label complexity, unlabeled sample complexity, and time complexity of the algorithm follows immediately from Theorem~\ref{thm:ap-bn}.
\end{proof}

For the adversarial noise setting,~\cite{ABL14} outlines an algorithm that returns a vector that has angle at most $\frac \pi 4$ with $u$. We state the algorithm in our context for completeness.

\begin{algorithm}
\caption{Master Algorithm in the Adversarial Noise Setting}
\begin{algorithmic}[1]
\REQUIRE Labeling oracle $\calO$, confidence $\delta$
\ENSURE a halfspace $\hat{v}$ such that $\theta(\hat{v}, u) \leq \frac{\pi}{4}$.
		\STATE $v_0 \leftarrow (1,0,\ldots,0)$.
		\STATE $v_+ \leftarrow \AP(\calO, v_0, \frac 1 {16},\frac{\delta}{3}, \cbr{m_k}, \cbr{b_k})$.
		\STATE $v_- \leftarrow \AP(\calO, -v_0, \frac 1 {16},\frac{\delta}{3}, \cbr{m_k}, \cbr{b_k})$.
		\STATE Define region $R := \cbr{x: \sign(v_+ \cdot x) \neq \sign(v_- \cdot x) }$.
		\STATE $S \gets$ Draw $8 \ln\frac{6}{\delta}$ iid examples from $D|_R$ and query their labels.
		\IF{$\err_S(h_{v_+}) \leq \err_S(h_{v_-})$}
			\RETURN $v_+$
		\ELSE
			\RETURN $v_-$
		\ENDIF
\end{algorithmic}
\label{alg:init-an}
\end{algorithm}

\begin{theorem}
Suppose Algorithm~\ref{alg:init-an} has inputs labeling oracle $\calO$ that satisfies
$\eta$-bounded noise condition with respect to $u$, confidence $\delta$,
sample schedule $\cbr{m_k}$
where $m_k = \Theta\del{ d (\ln d + \ln\frac{k}{\delta}) }$,
band width $\cbr{b_k}$ where $b_k = \tilde{\Theta}\del{\frac{ 2^{-k}}{\sqrt{d}}}$.
Then, with probability at least  $1-\delta$, the output $\hat{v}$ is such that
$\theta(\hat{v}, u) \leq \frac\pi 4$. Furthermore,
(1) the total number of label queries to oracle $\calO$ is at most
$\tilde{O}\del{ d }$; (2) the total number of unlabeled examples drawn is $\tilde{O}\del{d}$; (3) the algorithm runs in time $\tilde{O}\del{d^2}$.
\end{theorem}

The proof of this theorem is almost the same as Theorem~\ref{thm:init-bn} and is thus omitted.

\section{Basic Lemmas for the Upper Bounds}
\label{sec:basic}

In this section, we present a few useful lemmas that serve as the basis of proving Theorems~\ref{thm:ap-bn} and \ref{thm:ap-an}.

\subsection{Basic Facts}
We first collect a few useful facts for algebraic manipulations.

\begin{lemma}
\label{lem:power-exp}If $0\leq x\leq1-\frac{1}{e}$, then for any
$d\geq1$, $(1-\frac{x}{d})^{\frac{d}{2}}\geq e^{-x}\geq\frac{1}{2}$.
\end{lemma}
%\begin{proof}
%If $0\leq x\leq1-\frac{1}{e}$, then $\ln(1-x)\geq\frac{-e}{e-1}x$.
%Consequently, $(1-\frac{x}{d})^{\frac{d}{2}}=\exp\left(\frac{d}{2}\ln(1-\frac{x}{d})\right)\geq\exp\left(-\frac{d}{2}\frac{e}{e-1}\frac{x}{d}\right)\geq e^{-x}\geq\exp\left(-1+\frac{1}{e}\right)\geq\frac{1}{2}$.
%\end{proof}

\begin{lemma}
Given $a \in (0,\pi)$, if $x\in[0,a]$, then $\frac{\sin a}{a} x\leq\sin x\leq x$.
\label{lem:sin}
\end{lemma}
\begin{lemma}
\label{lem:cos}If $x\in[0,\pi]$, then $1-\frac{x^{2}}{2}\leq\cos x\leq1-\frac{x^{2}}{5}$.
\end{lemma}
\begin{lemma}
\label{lem:beta-func-lb-ub}
Let $\B(x,y) = \int_{0}^{1}(1-t)^{x-1}t^{y-1}\d t$ be the Beta function.
Then $\frac{2}{\sqrt{d-1}}\leq \B(\frac{1}{2},\frac{d}{2})\leq\frac{\pi}{\sqrt{d}}$.
\end{lemma}
%\begin{proof}
%$\B(\frac{1}{2},\frac{d}{2})=\B(\frac{1}{2},\frac{d}{2}-1)\frac{d-2}{d-1}$
%for $d>2$. The lemma follows by induction.
%\end{proof}

\subsection{Probability Inequalities}

\begin{lemma}[Azuma's Inequality]
\label{lem:azuma}
Let $\cbr{Y_t}_{t=1}^m$ be a bounded submartingale difference sequence, that is, $\E[Y_t|Y_1, \ldots, Y_{t-1}] \geq 0$, and
$\left|Y_{t}\right|\leq\sigma$. Then, with probability at least $1-\delta$,
\[ \sum_{t=1}^m Y_t \geq -\sigma \sqrt{2 m \ln\frac{1}{\delta}} \]
\end{lemma}

\begin{lemma}[Concentration of Geometric Random Variables]
Suppose $Z_1, \ldots, Z_n$ are iid geometric random variables with parameter $p$. Then,
\[ \P[Z_1 + \ldots + Z_n > \frac{2n}{p}] \leq \exp(-\frac{n}{4}) \]
\label{lem:geom-conc}
\end{lemma}
\begin{proof}
Since $Z_1 + \ldots + Z_n > \frac{2n}{p}$ implies that $Z_1 + \ldots + Z_n \geq \lceil\frac{2n}{p}\rceil$ (as $Z_1 + \ldots + Z_n$ is an integer), the left hand side is at most $\P[Z_1 + \ldots + Z_n \geq \lceil\frac{2n}{p}\rceil]$.

Let $X_1, \ldots, X_{\lceil \frac{2n}{p} \rceil}$ be a sequence of iid $\ber(p)$ random variables.
By standard relationship between Bernoulli random variables and geometric random variables,
we have that
\[ \P[Z_1 + \ldots + Z_n \geq \lceil\frac{2n}{p}\rceil] = \P[X_1 + \ldots + X_{\lceil \frac{2n}{p} \rceil - 1} \leq n - 1] \]
Note that $\P[X_1 + \ldots + X_{\lceil \frac{2n}{p} \rceil - 1} \leq n - 1] \leq \P[X_1 + \ldots + X_{\lceil \frac{2n}{p} \rceil} \leq n]$ since $X_{\lceil \frac{2n}{p} \rceil} \leq 1$.
Applying Chernoff bound, the above probability is at most
$\exp(-\lceil \frac{2n}{p} \rceil \cdot p \cdot \frac{1}{8}) \leq \exp(-\frac{n}{4})$.
\end{proof}

\subsection{Properties of the Uniform Distribution over the Unit Sphere}

\begin{lemma}[Marginal Density and Conditional Density]
If $(x_1, x_2, \ldots, x_d)$ is drawn from the uniform distribution over the unit
sphere, then:
\begin{enumerate}
\item $(x_1, x_2)$ has a density function of $p(z_1, z_2)$, where
			$p(z_1, z_2) = \frac{(1 - z_1^2 - z_2^2)^{\frac{d-4}{2}}}{\frac{2\pi}{d-2}}$.
\item Conditioned on $x_2 = b$, $x_1$ has a density function of $p_b(z)$, where
			$p_b(z) = \frac{(1 - b^2 - z^2)^{\frac{d-4}{2}}}{(1-b^2)^{\frac{d-3}{2}} \B(\frac{d-2}{2}, \frac{1}{2})}$.
\item $x_1$ has a density function of $p(z)$, where
			$p(z) = \frac{(1 - z^2)^{\frac{d-3}{2}}}{\B(\frac{d-1}{2}, \frac{1}{2})}$.
\end{enumerate}
\label{lem:marginal-conditional}
\end{lemma}

\begin{lemma}
Suppose $x$ is drawn uniformly from the unit sphere, and $b \leq \frac{1}{10\sqrt{d}}$.
Then,
$ \P\sbr{x_1 \in \intcc{\frac{b}{2}, b}} \geq \frac{\sqrt{d}}{8\pi} b$.
\label{lem:1d-lb}
\end{lemma}
\begin{proof}
\begin{eqnarray*}
	&& \P\sbr{x_1 \in \intcc{\frac{b}{2}, b}} \\
	&=&  \frac{\int_{b/2}^{b}(1 - t^2)^{\frac{d-3}{2}}\d t }{\B(\frac{d-1}{2}, \frac12)} \\
	&\geq& \frac{\frac{b}{2} (1 - b^2)^{\frac{d-3}{2}}}{\frac{\pi}{\sqrt{d-1}}} \geq \frac{\sqrt{d}}{8\pi} b
\end{eqnarray*}
where the first equality is from item 3 of Lemma~\ref{lem:marginal-conditional},
giving the exact probability density function of $x_1$, the first inequality
is from that $(1-t^2)^{\frac{d-3}{2}} \geq (1 - b^2)^{\frac{d-3}{2}}$ when $t \in \intcc{b/2, b}$,
and Lemma~\ref{lem:beta-func-lb-ub} giving upper bound on
$\B(\frac{d-1}{2}, \frac12)$, and the second inequality is from Lemma~\ref{lem:power-exp}
and that $d-1 \geq \frac d 2$.
\end{proof}

\begin{lemma}
Suppose $x$ is drawn uniformly from unit sphere restricted to the region $\cbr{x: v \cdot x = \xi}$,
and $u, v$ are unit vectors such that
$\theta(u,v) = \theta \in [0, \frac{9}{10}\pi]$ and $0 \leq \xi \leq \frac{\theta}{4\sqrt{d}}$.
Then,
\begin{enumerate}
	\item $\E[ u\cdot x ] \leq \xi$.
	\item $\E[ (u \cdot x)^2 ] \leq \frac{5\theta^2}{d}$.
	\item $\E[ (u\cdot x) \one\cbr{ u\cdot x < 0} ] \leq \xi - \frac{\theta}{36\sqrt{d}}$.
\end{enumerate}
\label{lem:cond-moment}
\end{lemma}
\begin{proof}
By spherical symmetry, without loss of generality,
let $v = (0,1,0,\ldots,0)$, and $u = (\sin\theta, \cos\theta, 0, \ldots, 0)$.
Let $x=(x_1, \ldots, x_d)$.
\begin{enumerate}
\item
\begin{eqnarray}
&&\E[ u\cdot x ] \nonumber \\
&=& \E[ x_1 \sin\theta + x_2 \cos\theta | x_2 = \xi ] \nonumber\\
&=& \E[x_1 | x_2 = \xi] \sin\theta + \xi \cos\theta \nonumber \\
&\leq& \xi \nonumber
\label{eqn:bn-progress-term1}
\end{eqnarray}
where the first two equalities are by algebra, the inequality follows from $\cos\theta \leq 1$ and
$\E[x_1 | x_2=\xi]=0$ since
the conditional distribution of $x_1$ given $x_2 = \xi$ is symmetric around
the origin.

\item
\begin{eqnarray*}
&&\E[ (u\cdot x)^2 ]  \\
&=& \E[ (x_1\sin\theta + x_2 \cos\theta)^2 | x_2 = \xi ] \\
&\leq& \E[ 2 x_1^2 \sin^2 \theta + 2 x_2^2 \cos^2 \theta | x_2 = \xi ] \\
&\leq& 2 \E[  x_1^2 | x_2 = \xi ] \sin^2 \theta + 2 \xi^2 \\
&\leq& 2\theta^2 \frac{\int_{-1}^1 z^2 (1-z^2)^{\frac{d-4}{2}} \d z}{\B(\frac{d-2}{2}, \frac12)} + 2 \xi^2 \\
&=& 2\theta^2 \frac{\B(\frac{d-2}{2}, \frac32)}{\B(\frac{d-2}{2}, \frac12)} + 2 \xi^2 \\
&\leq& \frac{5\theta^2}{d}
\end{eqnarray*}
where the first equality is by definition of $u$, the first inequality is from algebra
that $(A+B)^2 \leq 2A^2 + 2B^2$, the second inequality is from that
$|\cos\theta| \leq 1$, the third inequality is from item 2 of Lemma~\ref{lem:marginal-conditional}
and that $\sin \theta \leq \theta$, and the last inequality is from the fact that
$\frac{\B(\frac{d-2}{2}, \frac32)}{\B(\frac{d-2}{2}, \frac12)} = \frac{1}{d-1} \leq \frac{2}{d}$,
and $\xi^2 \leq \frac{\theta^2}{16d}$.
%= \frac{\Gamma(\frac{d-2}{2})\Gamma(\frac{3}{2}) / \Gamma(\frac{d+1}{2})}{\Gamma(\frac{d-2}{2})\Gamma(\frac{1}{2}) / \Gamma(\frac{d-1}{2})}

\item
\begin{eqnarray}
&&\E[ (u\cdot x) \one\cbr{u\cdot x < 0} ] \nonumber \\
&=& \E[ (x_1 \sin\theta + x_2 \cos\theta) \one\cbr{x_1 < -\xi \cot\theta} | x_2 = \xi ] \nonumber\\
&\leq& \E[ x_1 \one\cbr{x_1 < -\xi \cot\theta} | x_2 = \xi ] \sin\theta + \xi \nonumber \\
&=& \xi + \sin\theta \int_{-\sqrt{1-\xi^2}}^{-\xi\cot\theta} \frac{(1 - \xi^2 - x_1^2)^{\frac{d-4}{2}} x_1}{(1 - \xi^2)^{\frac{d-3}{2}}\B(\frac{d-2}{2}, \frac{1}{2})} \d x_1 \nonumber \\
&=& \xi - \sin\theta \frac{\frac{2}{d-2}\del{1 - \del{\frac{\xi}{\sin \theta}}^2}^{\frac{d-2}{2}}}{(1-\xi^2)^{\frac{d-3}{2}} \B(\frac{d-2}{2}, \frac12)} \nonumber \\
&\leq& \xi - \sin\theta \frac{2}{\pi\sqrt{d-2}} \del{1 - \del{\frac{\xi}{\sin \theta}}^2}^{\frac{d-2}{2}} \nonumber \\
&\leq& \xi - \frac{\sin \theta}{\pi\sqrt{d}} \nonumber \\
&\leq& \xi - \frac{\theta}{36\sqrt{d}} \nonumber
\label{eqn:bn-progress-term2}
\end{eqnarray}
where the first inequality is by algebra and $|\cos\theta| \leq 1$, the second equality is
by item 2 of Lemma~\ref{lem:marginal-conditional}, the third
 equality is by integration, the second inequality is from $(1-\xi^2)^{\frac{d-3}{2}} \leq 1$ and
 Lemma~\ref{lem:beta-func-lb-ub} that $\B(\frac{d-2}{2}, \frac12) \leq \frac\pi{\sqrt{d-2}}$,
 the third inequality follows by Lemma~\ref{lem:power-exp} that $\del{1 - \del{\frac{\xi}{\sin \theta}}^2}^{\frac{d-2}{2}} \geq \frac{1}{2}$,
 since $\xi \leq \frac{\theta}{4\sqrt{d}}$, and the last inequality follows from Lemma~\ref{lem:sin} that
$\sin\theta \geq \frac{5\theta}{18\pi}$ when $\theta \in [0, \frac9{10}\pi]$ and algebra.
\end{enumerate}
\end{proof}

\section{Proof of the Lower Bound}
\label{sec:pf-lb}
In this section, we give the proof of Theorem~\ref{thm:lb-bn} (label complexity lower bound
in the bounded noise setting).
The proof follows from two key lemmas, Lemma~\ref{lem:lb-d-eps} and Lemma~\ref{lem:lb-delta}.
We start with some additional definitions.

\begin{definition}
Let $\P,\Q$ be two probability measures on a common measurable space
and $\P$ is absolutely continuous with respect to $\Q$.
\begin{itemize}
\item The KL-divergence
between $\P$ and $\Q$ is defined as $D_{\text{KL}}\left(\P,\Q\right)=\E_{X\sim\P}\ln\frac{\P(X)}{\Q(X)}$.

\item We define $d_{\text{KL}}(p,q)=D_{\text{KL}}\left(\P,\Q\right)$,
where $\P,\Q$ are distributions of a Bernoulli($p$) and a Bernoulli($q$)
random variables respectively.

\item For random variables $X,Y,Z$, define the mutual information between
$X$ and $Y$ under $\P$ as $I(X;Y)=D_{\text{KL}}\left(\P(X,Y),\P(X)\P(Y)\right)=\E_{X,Y}\ln\frac{\P(X,Y)}{\P(X)P(Y)}$,
and define the mutual information between $X$ and $Y$ conditioned
on $Z$ under $\P$ as $I(X;Y\mid Z)=\E_{X,Y,Z}\ln\frac{\P(X,Y\mid Z)}{\P(X\mid Z)P(Y\mid Z)}$.

\item For a sequence of random variables ${X_1, X_2, \dots}$, denote by $X^n$ the subsequence $\cbr{X_1, X_2, \dots X_n}$.
\end{itemize}
\end{definition}
We will use the following two folklore information-theoretic lower bounds.

\begin{lemma}
\label{lem:Fano}Let $\calW$ be a class of parameters, and $\{P_{w}:w\in\calW\}$
be a class of probability distributions indexed by $\calW$ over some
sample space $\mathcal{X}$ . Let $d:\calW\times\calW\rightarrow\mathcal{\mathbb{R}}$
be a semi-metric. Let $\mathcal{V}=\left\{ w_{1},\dots,w_{M}\right\} \subseteq\calW$
such that $\forall i\neq j$, $d(w_{i},w_{j})\geq2s>0$. Let $V$
be a random variable uniformly taking values from $\mathcal{V}$,
and $X$ be drawn from $P_{V}$. Then for any algorithm $\calA$ that
given a sample $X$ drawn from $P_{w}$ outputs $\calA(X)\in\calW$,
the following inequality holds:
\[
\sup_{w\in\calW}P_{w}\left(d(w,\calA(X))\geq s\right)\geq1-\frac{I(V;X)+\ln2}{\ln M}
\]
\end{lemma}
\begin{proof}
For any algorithm $\calA$, define a test function $\hat{\Psi}:\mathcal{X}\rightarrow\{1,\dots,M\}$
such that
\[
\hat{\Psi}(X)=\arg\min_{i\in\{1,\dots,M\}}d(\calA(X),w_{i})
\]
We have
\[
\sup_{w\in\calW}P_{w}\left(d(w,\calA(X))\geq s\right)\geq\max_{w\in\mathcal{V}}P_{w}\left(d(w,\calA(X))\geq s\right)\geq\max_{i\in\{1,\dots,M\}}P_{w_{i}}\left(\hat{\Psi}(X)\neq i\right)
\]
The desired result follows by classical Fano's Inequality:
\[
\max_{i\in\{1,\dots,M\}}P_{w_{i}}\left(\hat{\Psi}(X)\neq i\right)\geq1-\frac{I(V;X)+\ln2}{\ln M}
\]
\end{proof}
\begin{lemma}
\label{lem:lb-ht}\citep[Lemma~5.1]{AB09} Let $\gamma\in(0,1)$, $\delta\in(0,\frac{1}{4})$,
$p_{0}=\frac{1-\gamma}{2}$, $p_{1}=\frac{1+\gamma}{2}$. Suppose
that $\alpha\sim$Bernoulli$(\frac{1}{2})$ is a random variable,
$\xi_{1},\dots,\xi_{m}$ are i.i.d. (given $\alpha$) Bernoulli$(p_{\alpha})$
random variables. If $m\leq2\left\lfloor \frac{1-\gamma^{2}}{2\gamma^{2}}\ln\frac{1}{8\delta(1-2\delta)}\right\rfloor $,
then for any function $f:\{0,1\}^{m}\rightarrow\{0,1\}$, $\P\left(f(\xi_{1},\dots,\xi_{m})\neq\alpha\right)>\delta$.
\end{lemma}
Next, we present two technical lemmas.
\begin{lemma}
\label{lem:packing}\citep[Lemma~6]{L95} For any $0<\gamma\leq\frac{1}{2}$, $d\geq1$,
there is a finite set $\calV\in\S^{d-1}$ such that the following two statements
hold:

1. For any distinct $w_{1},w_{2}\in\calV$, $\theta(w_{1},w_{2})\geq\pi\gamma$;

2. $|\calV|\geq\frac{\sqrt{d}}{2}\left(\frac{1}{2\pi\gamma}\right)^{d-1}-1$.
\end{lemma}
\begin{lemma}
\label{lem:kl-leq-xi-square}If $p\in[0,1]$ and $q\in(0,1)$, then
$d_{\text{KL}}(p,q)\leq\frac{(p-q)^{2}}{q(1-q)}$.
\end{lemma}
\begin{proof}
\begin{eqnarray*}
d_{\text{KL}}(p,q) & = & p\ln\frac{p}{q}+(1-p)\ln\frac{1-p}{1-q}\\
 & \leq & p(\frac{p}{q}-1)+ (1-p)(\frac{1-p}{1-q}-1)\\
 & = & \frac{(p-q)^{2}}{q(1-q)}
\end{eqnarray*}

where the inequality follows by $\ln x\leq x-1$.
\end{proof}
\begin{lemma}
\label{lem:lb-d-eps}For any $0\leq\eta<\frac{1}{2}$, $d>4$, $0<\epsilon\leq\frac{1}{4\pi}$,
$0<\delta<\frac{1}{2}$, for any active learning algorithm $\calA$, there is a $u\in\S^{d-1}$,
and a labeling oracle $\calO$ that satisfies $\eta$-bounded noise
condition with respect to $u$, such that if with probability at least $1-\delta$, $\calA$
makes at most $n$ queries to $\calO$ and outputs $v\in\S^{d-1}$
such that $\P[\sign(v \cdot x) \neq \sign(u \cdot x)]\leq\epsilon$,
then $n\geq\frac{d\ln\frac{1}{\epsilon}}{16(1-2\eta)^{2}}$.
\end{lemma}
\begin{proof}
We will prove this Lemma using Lemma~\ref{lem:Fano}.

First, we construct $\calW$, $\calV$, $d$, $s$, and $P_{\theta}$.
Let $\calW=\S^{d-1}$. Let $\calV$ be the set in Lemma~\ref{lem:packing}
with $\gamma=2\epsilon$. For any $w_{1},w_{2}\in\calW$,
let $d(w_{1},w_{2})=\theta(w_{1},w_{2})$, $s=\pi\epsilon$.
Fix any algorithm $\calA$. For any $w\in\calW$, any $x\in\calX$,
define $P_{w}[Y=1|X=x]=\begin{cases}
1-\eta, & w\cdot x\geq0\\
\eta, & w\cdot x<0
\end{cases}$, and $P_{w}[Y=0|X=x]=1-P_{w}[Y=1|X=x]$. Define $P_{w}^{n}$
to be the distribution of $n$ examples $\left\{ (X_{i},Y_{i})\right\} _{i=1}^{n}$
where $Y_{i}$ is drawn from distribution $P_{w}(Y|X_{i})$ and $X_{i}$
is drawn by the active learning algorithm $\calA$ based solely on
the knowledge of $\left\{ (X_{j},Y_{j})\right\} _{j=1}^{i-1}$.

By Lemma~\ref{lem:packing}, we have $M=\left|\calV\right|\geq\frac{\sqrt{d}}{2}\left(\frac{1}{4\pi\epsilon}\right)^{d-1}-1\geq\frac{1}{4}\left(\frac{1}{4\pi\epsilon}\right)^{d-1}$,
and $d(w_{1},w_{2})\geq2\pi\epsilon=2s$ for any distinct $w_{1},w_{2}\in\calV$.

Clearly, for any $w\in\calW$, if the optimal classifier is $w$,
and the oracle $\calO$ responds according to $P_{w}(\cdot\mid X=x)$,
then it satisfies $\eta$-bounded noise condition.
%Moreover, under bounded noise condition,
%for any $w\in\calW$, if $\err(w)-\err(u)\leq\epsilon$, then
%$d(w_{1},w_{2})\leq\frac{\pi\epsilon}{1-2\eta}$.
Therefore, to prove
the lemma, it suffices to show that if $n\leq \frac{d\ln\frac{1}{\epsilon}}{16(1-2\eta)^{2}} $,
then
\[
\sup_{w\in\calW}P_{w}\left(d(w,\calA(X^{n},Y^{n}))\geq s\right)\geq\frac12.
\]

Now, by Lemma~\ref{lem:Fano},
\[
\sup_{w\in\calW}P_{w}^{n}\left(d(w,\calA(X^{n},Y^{n}))\geq s\right)\geq1-\frac{I(V;X^{n},Y^{n})+\ln2}{\ln M}\geq1-\frac{I(V;X^{n},Y^{n})+\ln2}{(d-1)\ln\frac{1}{4\pi\epsilon}-\ln4}.
\]

It remains to show if $n=\frac{d\ln\frac{1}{\epsilon}}{16(1-2\eta)^{2}}$,
then $I(V;X^{n},Y^{n})\leq\frac12\left((d-1)\ln\frac{1}{4\pi\epsilon}-\ln4\right)-\ln2$.

By the chain rule of mutual information, we have

\[
I(V;X^{n},Y^{n})=\sum_{i=1}^{n}\left(I\left(V;X_{i}\mid X^{i-1},Y^{i-1}\right)+I\left(V;Y_{i}\mid X^{i},Y^{i-1}\right)\right)
\]

First, we claim $V$ and $X_{i}$ are conditionally independent given
$\left\{ X^{i-1},Y^{i-1}\right\} $, and thus\\ $I\left(V;X_{i}\mid X^{i-1},Y^{i-1}\right)=0$.
The proof for this claim is as follows. Since the selection of $X_i$ only depends on
algorithm $\calA$ and $X^{i-1}, Y^{i-1}$,
for any $v_{1},v_{2}\in\calV$, $\P\left(X_{i}\mid v_{1},X^{i-1},Y^{i-1}\right)=\P\left(X_{i}\mid v_{2},X^{i-1},Y^{i-1}\right)$.
Thus,
\begin{eqnarray*}
\P\left(X_{i}\mid X^{i-1},Y^{i-1}\right) & = & \sum_{v}\P\left(X_{i},v\mid X^{i-1},Y^{i-1}\right)\\
 & = & \sum_{v}\P(v)\P\left(X_{i}\mid v,X^{i-1},Y^{i-1}\right)\\
 & = & \frac{1}{\left|\calV\right|}\sum_{v}\P\left(X_{i}\mid v,X^{i-1},Y^{i-1}\right)\\
 & = & \P\left(X_{i}\mid V,X^{i-1},Y^{i-1}\right)
\end{eqnarray*}

Next, we show $I\left(V;Y_{i}\mid X^{i},Y^{i-1}\right)\leq5(1-2\eta)^{2}\ln2.$
On one hand, since $Y_{i}\in\left\{ -1,+1\right\} $,
$I\left(V;Y_{i}\mid X^{i},Y^{i-1}\right) \leq H \left(V \mid X^{i},Y^{i-1}\right) \leq  \ln2$.
where $H(\cdot|\cdot)$ is the conditional entropy.

On the other hand,
\begin{align*}
 & I\left(V;Y_{i}\mid X^{i},Y^{i-1}\right)\\
= & \E_{X^{i},Y^{i},V}\left[\ln\frac{\P\left(V,Y_{i}\mid X^{i},Y^{i-1}\right)}{\P\left(V\mid X^{i},Y^{i-1}\right)\P\left(Y_{i}\mid X^{i},Y^{i-1}\right)}\right]\\
= & \E_{X^{i},Y^{i},V}\left[\ln\frac{\P\left(Y_{i}\mid V,X^{i},Y^{i-1}\right)}{\P\left(Y_{i}\mid X^{i},Y^{i-1}\right)}\right]\\
= & \E_{X^{i},Y^{i},V}\left[\ln\frac{\P\left(Y_{i}\mid V,X^{i},Y^{i-1}\right)}{\E_{V'}\P\left(Y_{i}\mid V',X^{i},Y^{i-1}\right)}\right]\\
\leq & \E_{X^{i},Y^{i},V,V'}\left[\ln\frac{\P\left(Y_{i}\mid V,X^{i},Y^{i-1}\right)}{\P\left(Y_{i}\mid V',X^{i},Y^{i-1}\right)}\right]\\
\leq & \max_{x^{i},y^{i-1},v,v'}D_{\text{KL}}\left(\P\left(Y_{i}\mid x^{i},y^{i-1},v\right),\P\left(Y_{i}\mid x^{i},y^{i-1},v'\right)\right)\\
= & \max_{x^{i},y^{i-1},v,v'}D_{\text{KL}}\left(\P\left(Y_{i}\mid x_{i},v\right),\P\left(Y_{i}\mid x_{i},v'\right)\right)\\
= & \max_{x^{i},v,v'}D_{\text{KL}}\left(P_{v}\left(Y_{i}\mid x_{i}\right),P_{v'}\left(Y_{i}\mid x_{i}'\right)\right)\\
\leq & \frac{(1-2\eta)^{2}}{\eta(1-\eta)}
\end{align*}

where the first inequality follows from the convexity of KL-divergence,
and the last inequality follows from Lemma~\ref{lem:kl-leq-xi-square}.

Combining the two upper bounds, we get $I\left(V;Y_{i}\mid X^{i},Y^{i-1}\right)\leq\min\left\{ \ln2,\frac{(1-2\eta)^{2}}{\eta(1-\eta)}\right\} \leq5(1-2\eta)^{2}\ln2$.

Therefore, $I(V;X^{n},Y^{n})\leq5n(1-2\eta)^{2}\ln2$. If $n\leq\frac{d\ln\frac1{\epsilon}}{16(1-2\eta)^{2}}\leq\frac{\frac 1 2\left((d-1)\ln\frac{1}{4\pi\epsilon}-\ln4\right)-\ln2}{5(1-2\eta)^{2}\ln2}$,
then $I(V;X^{n},Y^{n})\leq\frac12\left((d-1)\ln\frac{1}{4\pi\epsilon}-\ln4\right)-\ln2.$
This concludes the proof.
\end{proof}
\begin{lemma}
\label{lem:lb-delta}For any $d>0$, $0\leq\eta<\frac{1}{2}$, $0<\epsilon<\frac{1}{3}$,
$0<\delta\leq\frac{1}{4}$, for any active learning algorithm $\calA$, there is a
$u\in\S^{d-1}$, and a labeling oracle $\calO$ that satisfies $\eta$-bounded
noise condition with respect to $u$, such that if with probability at least $1-\delta$,
$\calA$ makes at most $n$ queries to $\calO$ and outputs $v\in\S^{d-1}$
such that $\P[\sign(v \cdot x) \neq \sign(u \cdot x)]\leq\epsilon$,
then $n\geq\Omega\left(\frac{\eta\ln\frac{1}{\delta}}{(1-2\eta)^{2}}\right)$.
\end{lemma}
\begin{proof}
We prove this result by reducing the hypothesis testing problem in
Lemma~\ref{lem:lb-ht} to our problem of learning halfspaces.

Fix $d,\epsilon,\delta,\eta$. Suppose $\calA$ is an algorithm that
for any $u\in\S^{d-1}$, under $\eta$-bounded noise condition, with probability
at least $1-\delta$ outputs $v\in\S^{d-1}$ such that
$\P[\sign(v \cdot x) \neq \sign(u \cdot x)]\leq\epsilon<\frac13$,
which implies $\theta(v,u)\leq\frac{\pi}{3}$
under bounded noise condition.

Let $p_{0}=\eta$, $p_{1}=1-\eta$. Suppose that $\alpha\sim$Bernoulli$(\frac{1}{2})$
is an unknown random variable. We are given a sequence of i.i.d. (given
$\alpha$) Bernoulli$(p_{\alpha})$ random variables $\xi_{1},\xi_{2}\dots$,
and would like to test if $\alpha$ equals $0$ or $1$.

Define $e=(1,0,0,\dots,0)\in\R^{d}$. Construct a labeling oracle
$\calO$ such that for the $i$-th query $x_{i}$, it returns $2\xi_{i}-1$ if
$x_{i}\cdot e\geq0$, and $1-2\xi_{i}$ otherwise. Clearly, the oracle $\calO$ satisfies $\eta$-bounded noise condition with respect to underlying halfspace
$u=(2\alpha-1)e = (2\alpha-1,0,0,\dots,0)\in\R^{d}$.

Now, we run learning algorithm $\calA$ with oracle $\calO$. Let
$m$ be the number of queries $\calA$ makes, and $\calA(\xi_{1},\dots,\xi_{m})$
be the normal vector of the halfspace output by the learning algorithm.
We define
\[
f(\xi_{1},\dots,\xi_{m})=
\begin{cases}
0 & \text{if }\calA(\xi_{1},\dots,\xi_{m})\cdot e<0\\
1 & \text{otherwise}
\end{cases}.
\]

By our assumption of $\calA$ and construction of $\calO$, $\P\left(\theta\left(u,\calA(\xi_{1},\dots,\xi_{m})\right)\leq\frac{1}{3}\pi\right)\geq1-\delta$,
so $\P\left(f(\xi_{1},\dots,\xi_{m})=\alpha\right)\geq1-\delta$, implying $\P\left(f(\xi_{1},\dots,\xi_{m})\neq\alpha\right)\leq\delta$.
By Lemma~\ref{lem:lb-ht}, $m\geq2\left\lfloor \frac{4\eta(1-\eta)}{(1-2\eta)^{2}}\ln\frac{1}{8\delta(1-2\delta)}\right\rfloor =\Omega\left(\frac{\eta\ln\frac{1}{\delta}}{(1-2\eta)^{2}}\right)$.
\end{proof}